\documentclass[10pt]{article}
\usepackage[top=1in, bottom=1in, left=0.9in, right=0.9in]{geometry}

\usepackage{custom_style}

\title{\textbf{The Benefit of Being Bayesian in Online Conformal Prediction}}
\author{
Zhiyu Zhang\\
Harvard University\\
\texttt{zhiyuz@seas.harvard.edu}\\
  \and
Zhou Lu \\
Princeton University\\
\texttt{zhoul@princeton.edu} \\
  \and
Heng Yang\\
Harvard University\\
\texttt{hankyang@seas.harvard.edu}\\
}
\date{}

\begin{document}
\maketitle

\begin{abstract}
Based on the framework of Conformal Prediction (CP), we study the online construction of confidence sets given a black-box machine learning model. By converting the target confidence levels into quantile levels, the problem can be reduced to predicting the quantiles (in hindsight) of a sequentially revealed data sequence. Two very different approaches have been studied previously:
\begin{itemize}
    \item Assuming the data sequence is iid or exchangeable, one could maintain the empirical distribution of the observed data as an algorithmic belief, and directly predict its quantiles. 
    \item Due to the fragility of statistical assumptions, a recent trend is to consider the non-distributional, adversarial setting and apply first-order online optimization algorithms to moving quantile losses. However, it requires the oracle knowledge of the target quantile level, and suffers from a previously overlooked monotonicity issue due to the associated loss linearization.
\end{itemize}

This paper presents an adaptive CP algorithm that combines their strengths. Without any statistical assumption, it is able to answer multiple arbitrary confidence level queries with low regret, while also overcoming the monotonicity issue suffered by first-order optimization baselines. Furthermore, if the data sequence is actually iid, then the same algorithm is automatically equipped with the ``correct'' coverage probability guarantee. 

To achieve such strengths, our key technical innovation is to regularize the aforementioned algorithmic belief (the empirical distribution) by a Bayesian prior, which robustifies it by simulating a non-linearized Follow the Regularized Leader (FTRL) algorithm on the output. Such a belief update backbone is shared by prediction heads targeting different confidence levels, bringing practical benefits analogous to the recently proposed concept of U-calibration \citep{kleinberg2023u}.
\end{abstract}

\section{Introduction}\label{section:intro}

\begin{wrapfigure}{R}{0.41\textwidth}
\vspace{-10pt}
    \centering
\includegraphics[width=0.4\textwidth]{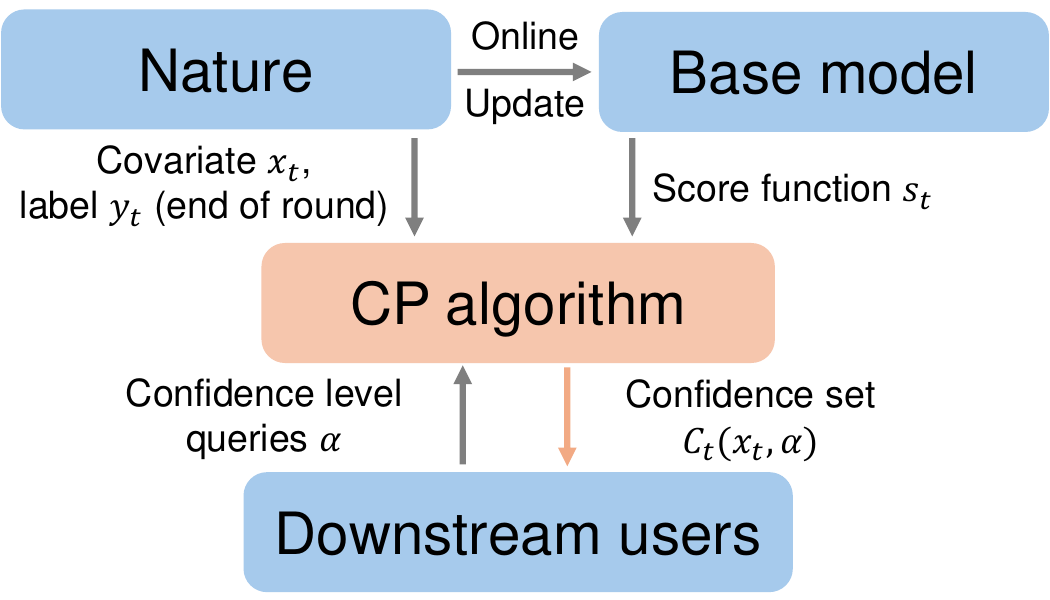}
    \caption{The CP interaction protocol.}
    \label{fig:Diagram}
    \vspace{-10pt}
\end{wrapfigure}

Modern machine learning (ML) models are better at point prediction compared to probabilistic prediction. For example, when given an image classification task, they are better at responding ``\emph{this image is most likely a white cat}'', rather than ``\emph{I'm 90\% sure this image is an animal, 60\% sure it's a cat, and 30\% sure it's a white cat}''. For downstream users, the more nuanced probabilistic predictions are often important for risk assessment. The challenge, however, lies in aligning the model's own uncertainty evaluation with its actual performance in the real world. 

\emph{Conformal prediction} (CP) \citep{vovk2005algorithmic} has recently emerged as a premier framework to address this challenge, as it blends the empirical strength of modern ML with the theoretical soundness of traditional statistical methods. As illustrated in Figure~\ref{fig:Diagram}, CP algorithms make \emph{confidence set predictions} on the label space, by sequentially interacting with three other parties: the \emph{nature} (i.e., the data stream), a \emph{black-box ML model}, and \emph{downstream users}. Concretely, let $\calX$ and $\calY$ be the covariate space and the label space. In each (the $t$-th) round, we study the following interaction protocol. 
\begin{enumerate}
\item We, as the CP algorithm, observe a \emph{target covariate} $x_t\in\calX$ from the nature, and a \emph{score function} $s_t:\calX\times\calY\rightarrow [0,R]$ generated by a black-box ML model called \textsc{Base}.
\item The downstream users select a finite set of \emph{confidence level queries}, $A_t\subset [0,1]$.
\item Given each confidence level query $\alpha\in A_t$, we predict a \emph{score threshold} $r_t(x_t,\alpha)$ based on existing observations, which leads to a \emph{confidence set prediction}\footnote{Without loss of generality, we assume the score function $s_t$ is \emph{negatively oriented}: smaller $s_t(x_t,y)$ means the ML model is more certain that the candidate label $y$ is the true label $y_t$; e.g., \citep{romano2020classification}.}
\begin{equation}\label{eq:cp}
\calC_t(x_t,\alpha)=\left\{y\in\calY:s_t(x_t,y)\leq r_t(x_t,\alpha)\right\}.
\end{equation}
\item Nature reveals the \emph{ground truth label} $y_t\in\calY$, and thus the \emph{true score} $r^*_t\defeq s_t(x_t,y_t)$, to us. 
\item The $(x_t,y_t)$ pair is passed to \textsc{Base}, which it optionally uses to generate the score function $s_{t+1}$. 
\end{enumerate}

The intuition is that by sequentially evaluating \textsc{Base} on the target data, we generate better score thresholds that ``correct'' the uncertainty evaluation from \textsc{Base} itself. Our goal is thus clear in a very broad sense -- predicting confidence sets with theoretically guaranteed \emph{validity}. Say if a user only queries the confidence level $\alpha=90\%$, then our CP algorithm needs to provide certain quantitative evidence that incentivizes the user to treat $\calC_t(x_t,\alpha)$ as the $90\%$ confidence set about the true label $y_t$. The same should hold true when the user queries $\alpha=80\%, 95\%, 99\%,\ldots$, or all of these values at the same time. While this is well-studied under various statistical assumptions introduced later, the present work is mainly about the \emph{adversarial}, game-theoretic setting where no statistical assumption is imposed at all. 

\subsection{Background}\label{subsection:background}

We first present the necessary background before our main results. Starting from the simplest case, let us assume that ($i$) within a time horizon $T$, the collection of true scores $r^*_{1:T}$ are iid samples of a random variable $X$ with strictly positive density, and ($ii$) the $\alpha$-\emph{quantile} of $X$,\footnote{The $\alpha$-quantile of a real random variable $X$ is defined as $q_{\alpha}(X)\defeq\min\{x:\P(X\leq x)\geq \alpha\}$.} denoted by $q_{\alpha}(X)$, is known. A natural strategy is thus predicting $r_t(x_t,\alpha)=q_{\alpha}(X)$, which ensures that the \emph{coverage condition} $y_t\in\calC_t(x_t,\alpha)$ holds with probability exactly $\alpha$. This certifies the validity of $\calC_t(x_t,\alpha)$ in the strongest probabilistic sense. 

Although the imposed assumptions are clearly unrealistic, this example illustrates a central principle of CP: the predicted score threshold $r_t(x_t,\alpha)$ should ideally be the $\alpha$-quantile of \emph{some} distribution of $r^*_{1:T}$. A key challenge of CP is thus generalizing this principle to more realistic settings, as described below. 
\begin{itemize}
\item \emph{Direct approach:} Still assuming the sequence $r^*_{1:T}$ is iid but the population quantile $q_{\alpha}(X)$ is unknown, we could instead estimate $q_{\alpha}(X)$ on the fly. Specifically, our algorithm maintains the empirical distribution of $r^*_{1:t-1}$, denoted by $P_t=\bar P(r^*_{1:t-1})$, as an \emph{algorithmic belief} about the unknown distribution of $X$. When queried with an arbitrary confidence level $\alpha$, the algorithm ``post-processes'' the belief by setting $r_t(x_t,\alpha)=q_{\alpha}(P_t)$. In the standard terms of statistical learning, this is equivalent to \emph{Empirical Risk Minimization} (ERM) with the \emph{quantile loss} $l_\alpha(r,r^*)\defeq(\bm{1}[r\geq r^*]-\alpha)(r-r^*)$, i.e., 
\begin{equation}\label{eq:erm}
r_t(x_t,\alpha)=q_{\alpha}(P_t)\in\argmin_{r\in[0,R]}\sum_{i=1}^{t-1}l_\alpha(r,r^*_i).
\end{equation}
A well-known refinement of this approach called \emph{Split Conformal} \citep{papadopoulos2002inductive} sets $r_t(x_t,\alpha)=q_{\alpha+o(1)}(P_t)$, where the $o(1)$ offset (with respect to $t\goto\infty$) ensures that even under a relaxation of the iid condition called \emph{exchangeability}, a suitable notion of coverage probability is lower bounded by $\alpha$ \citep{roth2022uncertain}.

\item \emph{Indirect approach:} However, even the iid or exchangeability assumption is often too strong in practice, e.g., when \textsc{Base} itself is updated using gradient descent in Step 5 of the interaction protocol. Motivated by this challenge, a recent trend \citep{gibbs2021adaptive} is to remove all statistical assumptions, and instead estimate the empirical quantile of $r^*_{1:T}$ using first-order optimization algorithms from \emph{adversarial online learning} \citep{hazan2023introduction,orabona2023modern}. Taking gradient descent for example, such an approach amounts to picking an initialization $r_1(x_1,\alpha)\in[0,R]$ and following with the incremental update
\begin{equation}\label{eq:first_order_update}
r_{t+1}(x_{t+1},\alpha)=r_t(x_t,\alpha)-\eta_t\partial l_\alpha(r_t(x_t,\alpha),r^*_t),
\end{equation}
where $\eta_t>0$ is the \emph{learning rate}, and $\partial l_\alpha(r,r^*)$ can be any subgradient of the quantile loss $l_\alpha$ with respect to the first argument. Due to the \emph{absence of probability} in this setting, alternative performance metrics have to be considered, such as the \emph{post-hoc coverage frequency} and the \emph{regret}.
\end{itemize}

How do these two approaches compare? Although the approach of first-order optimization does not need statistical assumptions, it requires being ``iterate-centric'' rather than ``data-centric'': one has to fix a single confidence level $\alpha$ beforehand, and the predicted threshold $r_t(x_t,\alpha)$ depends on how previous predictions $r_1(x_1,\alpha),\ldots,r_{t-1}(x_{t-1},\alpha)$ compare to the true scores $r^*_{1:t-1}$, rather than just $r^*_{1:t-1}$ itself. This leads to a previously overlooked \emph{monotonicity issue} that undermines the trustworthiness of the obtained confidence sets: 
\begin{itemize}
\item Two copies of an algorithm with $\alpha_1<\alpha_2$ can output $\calC_t(x_t,\alpha_2)\subsetneq \calC_t(x_t,\alpha_1)$, even if the initializations are the same. That is, the higher-confidence set is strictly smaller, violating the monotonicity requirement of probability measures. See Section~\ref{section:validity} for details. 
\end{itemize}

In contrast, the direct ERM approach does not suffer from this issue, but the problem is that being equivalent to \emph{Follow the Leader} (FTL) in online learning, it is well-known that ERM can suffer the vacuous $\Omega(T)$ regret on adversarial quantile losses. This motivates the important question: 
\begin{center}
\textit{Is there an adaptive CP algorithm that enjoys the best of both worlds?}
\end{center}

\subsection{Our Result}\label{subsection:contribution}

We answer this question in the affirmative, by presenting a single CP algorithm with the following strengths:
\begin{itemize}
\item Just like the ERM approach, it can answer multiple arbitrary confidence level queries online. 
\item Without any statistical assumption, it guarantees the optimal regret bound
\begin{equation*}
\sum_{t=1}^Tl_\alpha(r_t(x_t,\alpha),r^*_t)-\sum_{t=1}^Tl_\alpha(q_{\alpha}(r^*_{1:T}),r^*_t)=O(R\sqrt{T}),
\end{equation*}
simultaneously for all time horizon $T$, true score sequence $r^*_{1:T}$, and confidence level $\alpha\in[0,1]$. Notice that the comparator $q_{\alpha}(r^*_{1:T})$ would be a natural fixed prediction had one known the empirical distribution of the true score sequence $r^*_{1:T}$ beforehand.
\item Unlike first-order optimization baselines, it does not suffer from the aforementioned monotonicity issue, due to being ``data-centric'' rather than ``iterate-centric''. 
\item If the true scores $r^*_{1:T}$ are indeed iid, then it automatically achieves almost the same guarantees, including the \emph{dataset-conditional coverage probability} and the \emph{excess quantile risk}, as the ERM baseline. That is, the proposed algorithm can \emph{adapt} to the easier statistical setting of CP, despite being designed for the more general adversarial setting. 
\end{itemize}

From a technical perspective, our algorithm is a remarkably simple Bayesian analogue of the ERM approach. Instead of setting its algorithmic belief as the empirical distribution of the past, $P_t=\bar P(r^*_{1:t-1})$, we set it as the convex combination
\begin{equation}\label{eq:belief_first}
P_t=\lambda_t P_0+(1-\lambda_t)\bar P(r^*_{1:t-1}),
\end{equation}
where $P_0$ is a prior, and $\lambda_{t}\in[0,1]$ is a hyperparameter. The key observations are two-fold.
\begin{itemize}
\item If $r^*_{1:t-1}$ are iid samples of an underlying distribution, then Eq.(\ref{eq:belief_first}) is the posterior mean of a \emph{Bayesian distribution estimator} based on the \emph{Dirichlet process,} whereas the existing choice of $P_t=\bar P(r^*_{1:t-1})$ corresponds to the frequentist distribution estimator. Choosing $\lambda_t=O(t^{-1/2})$ leads to the state-of-the-art statistical guarantees. 
\item Without any statistical assumption, the Bayesian version of $P_t$ offers the advantage of \emph{downstream regularization}: the associated score threshold $r_t(x_t,\alpha)=q_{\alpha}(P_t)$ is equivalent to the output of a non-linearized \emph{Follow the Regularized Leader} (FTRL) algorithm from adversarial online learning. The desirable regret bound naturally follows, and the monotonicity issue is resolved as well. 
\end{itemize}

\subsection{Related Work}\label{subsection:related}

For the background of CP and its numerous applications, the readers are referred to several excellent resources \citep{vovk2005algorithmic,roth2022uncertain,angelopoulos2023conformal,tibshirani2023advanced}. 

\paragraph{Online CP} Our work contributes to the growing body of literature on online adversarial conformal prediction, initiated by \citep{gibbs2021adaptive}. They showed that running gradient descent with constant learning rate can achieve low \emph{coverage frequency error}, i.e., 
\begin{equation}\label{eq:coverage}
\abs{\alpha-T^{-1}\sum_{t=1}^T\bm{1}[r^*_t\leq r_t(x_t,\alpha)]}=o(1),
\end{equation}
and its sliding-window analogues. \citep{bastani2022practical} demonstrated an intriguing weakness of this performance metric: 
one could trivially satisfy this coverage frequency bound by predicting a data-independent alternation between the empty set and the entire label space. To rule out such pathological cases, it has thus been standard to consider an additional performance metric, such as the regret \citep{bhatnagar2023improved,gibbs2024conformal,zhang2024discounted} or the \emph{multi-calibrated} coverage frequency \citep{bastani2022practical}. Under the additional iid assumption, \citep{angelopoulos2024online} studied the asymptotic coverage probability achieved by gradient descent; in contrast, our best-of-both-world guarantees are all non-asymptotic. 

Different from the literature, the present work emphasizes regret minimization in online CP, as we believe it offers advantages \emph{even over simultaneously bounding the regret and the coverage frequency error} (since loss linearization is not necessary anymore). See Section~\ref{subsection:adversarial} for a detailed discussion. 

\paragraph{Bayesian uncertainty quantification} Our work is also drastically different from the typical uncertainty quantification pipeline in \emph{Bayesian machine learning} \citep{neal2012bayesian}, where one directly constructs posteriors on the parameter space and then uses it for confidence set predictions. While such a pipeline often suffers from scalability issues, we embed the construction of posteriors into the conformal prediction framework, thereby making the resulting algorithm both modular (i.e., can operate on top of arbitrary black-box ML models) and computationally efficient. Related but different from our focus, \cite{fong2021conformal} studied a CP problem where the base ML model itself is Bayesian. 

\paragraph{Adversarial Bayes} A common challenge for Bayesian algorithms is the choice of the prior. In this work, we exploit the algorithmic equivalence of Bayesian priors and regularizers, thereby utilizing the systematic study of regularization in adversarial online learning to demonstrate that \emph{the uniform prior is minimax optimal} for regret minimization. That is, one could avoid tuning the prior in a theoretically justified manner. Broadly speaking, such an idea of ``adversarial Bayes'' is closely related to a classical algorithm family in online learning called \emph{Follow the Perturbed Leader} (FTPL) \citep{kalai2005efficient}. Notable examples of FTPL include \emph{Thompson sampling} \citep{thompson1933likelihood,lattimore2020bandit,xu2023bayesian}, a prevalent Bayesian approach for bandits and reinforcement learning, and \emph{U-calibration} \citep{kleinberg2023u,luo2024optimal}, a recently proposed framework for loss-agnostic decision making. Despite being deterministic, our approach resembles the high level idea of U-calibration. See Section~\ref{subsection:adversarial} for a detailed discussion of their connections and differences.

Besides U-calibration, \emph{omniprediction}  \citep{gopalan2022omnipredictors,garg2024oracle} is another iconic framework for loss-agnostic decision making, whose main idea is to maintain a \emph{multi-calibrated} algorithmic belief in the sense of \citep{hebert2018multicalibration}. Our approach does not require calibration as an underlying mechanism. 

\subsection{Notation}

This paper studies the \emph{marginal} setting of CP, which means the threshold prediction $r_t(x_t,\alpha)$ will be independent of $x_t$; therefore we write it as $r_t(\alpha)$ for conciseness. For any symbol $x$, $x_{1:t}$ (e.g., $r^*_{1:t}$) represents the tuple $[x_1,\ldots,x_t]$. $\bar P(\cdot)$ denotes the empirical distribution of its input, and $q_{\alpha}(\cdot)$ denotes the $\alpha$-quantile. Our regret bound concerns the quantile (or pinball) loss defined as $l_\alpha(r,r^*)\defeq(\bm{1}[r\geq r^*]-\alpha)(r-r^*)$. $\log$ denotes the natural logarithm. 

\section{The Need for Monotonicity}\label{section:validity}

To begin with, we use a numerical experiment to elaborate the monotonicity issue suffered by existing adversarial online CP algorithms. This has been overlooked in the literature to the best of our knowledge, as all the existing approaches we are aware of require fixing a single target confidence level $\alpha$ at the beginning of the CP game. Code is available at \url{https://github.com/zhiyuzz/Bayesian-Conformal/blob/main/QuantilePrediction_IID.ipynb}.

Specifically, we consider two baselines, \emph{Online Gradient Descent} (OGD) from \citep{gibbs2021adaptive}, and \emph{MultiValid Prediction} (MVP) from \citep{bastani2022practical}. To support multiple confidence level queries, we adopt the following nearest-neighbor routing on top of their independent copies. It consists of three steps: 
\begin{enumerate}
\item Evenly discretize the $[0,1]$ interval of possible confidence levels using a grid $\tilde A$.
\item For each $\tilde \alpha\in\tilde A$, maintain a copy of the considered baseline (OGD or MVP) targeting $\tilde\alpha$.
\item When queried with $\alpha$, follow the output of the base algorithm (from Step 2) targeting the nearest neighbor of $\alpha$ in $\tilde A$. 
\end{enumerate}
Initiated with OGD and MVP, the resulting algorithms are named as MultiOGD and MultiMVP respectively. 

In the experiment, we fix $R=1$. The true score sequence $r^*_{1:T}$ is sampled iid from the uniform distribution on $[0,1]$, and we evaluate the thresholds $r_{1:T}(\alpha)$ predicted by different CP algorithms, under different $\alpha$ values. For each OGD baseline targeting $\tilde \alpha$, we use the standard learning rate $\eta_t=t^{-1/2}$, and initialize it with $r_1(\tilde \alpha)=\tilde\alpha$. The base MVP algorithms are all initialized at 0 following \citep{bastani2022practical}. The point is that the initialization of MultiOGD and MultiMVP cannot cause any monotonicity violation. For comparison, we also test ERM as well as our Bayesian algorithm to be introduced in Section~\ref{section:main}.

\begin{figure}[ht]
    \centering
    \begin{minipage}{0.24\textwidth}
        \centering
        \includegraphics[width=\textwidth]{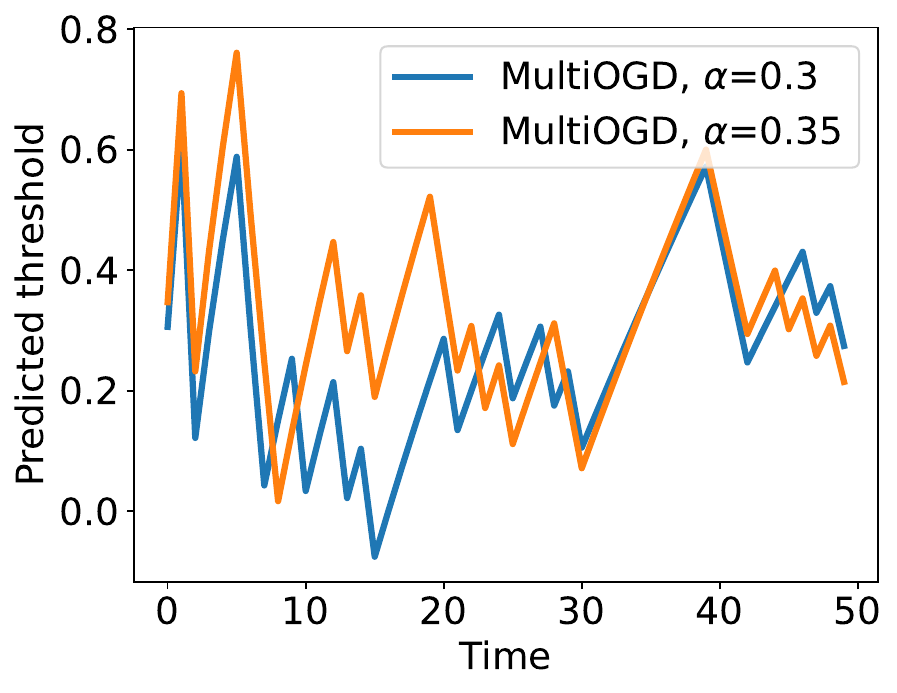}
    \end{minipage}
    \hfill
    \begin{minipage}{0.24\textwidth}
        \centering
        \includegraphics[width=\textwidth]{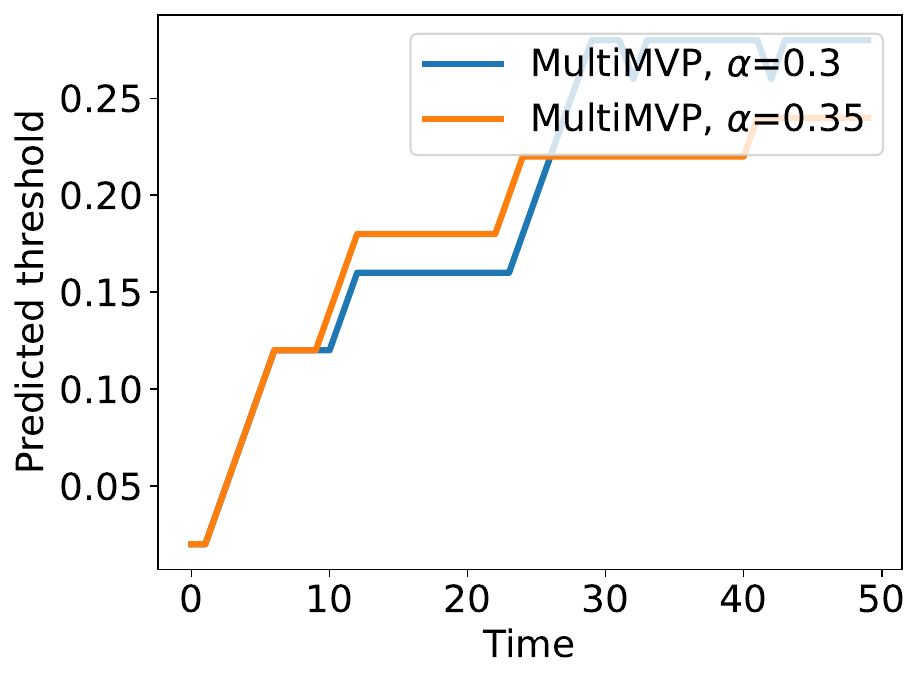}
    \end{minipage}
    \hfill
    \begin{minipage}{0.24\textwidth}
        \centering
        \includegraphics[width=\textwidth]{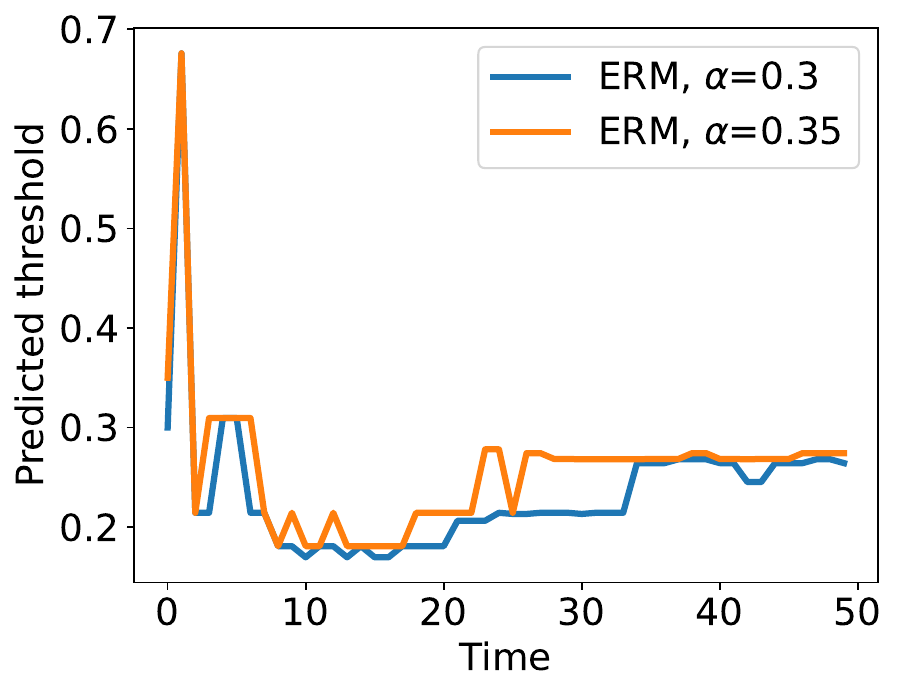}
    \end{minipage}
    \hfill
    \begin{minipage}{0.24\textwidth}
        \centering
        \includegraphics[width=\textwidth]{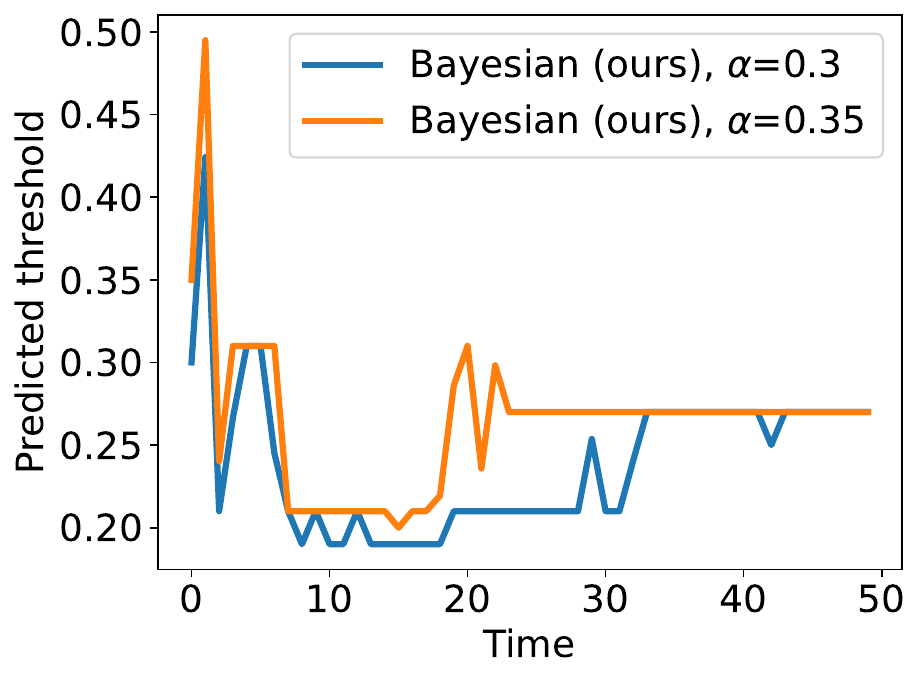}
    \end{minipage}

    \begin{minipage}{0.24\textwidth}
        \centering
        \includegraphics[width=\textwidth]{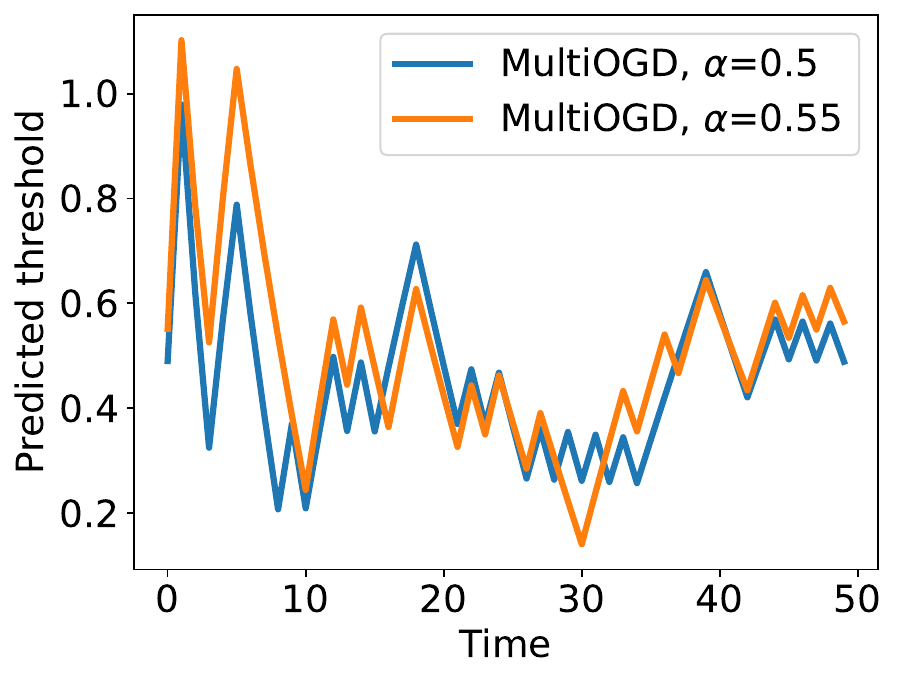}
    \end{minipage}
    \hfill
    \begin{minipage}{0.24\textwidth}
        \centering
        \includegraphics[width=\textwidth]{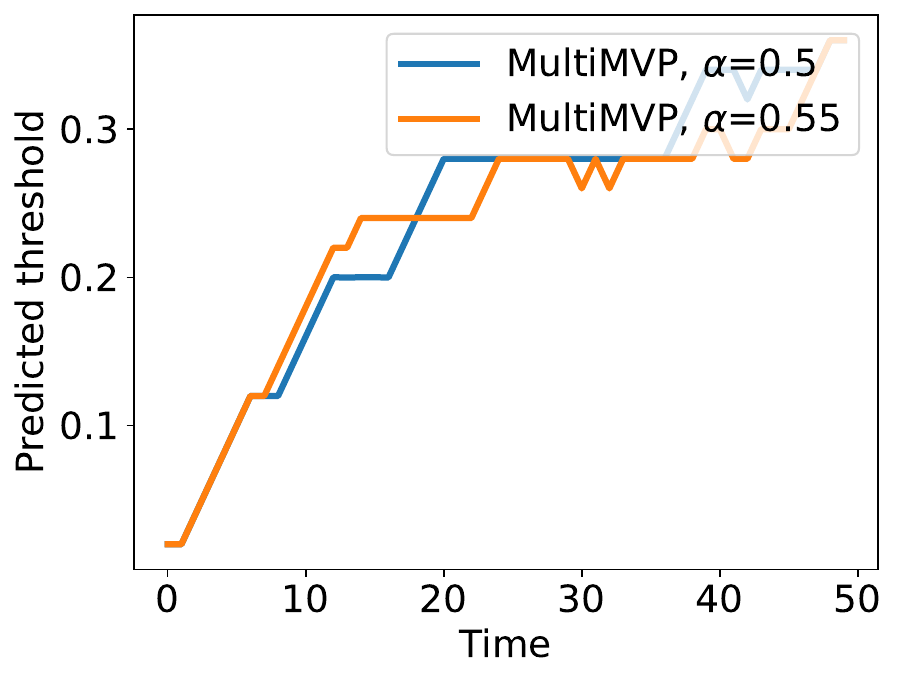}
    \end{minipage}
    \hfill
    \begin{minipage}{0.24\textwidth}
        \centering
        \includegraphics[width=\textwidth]{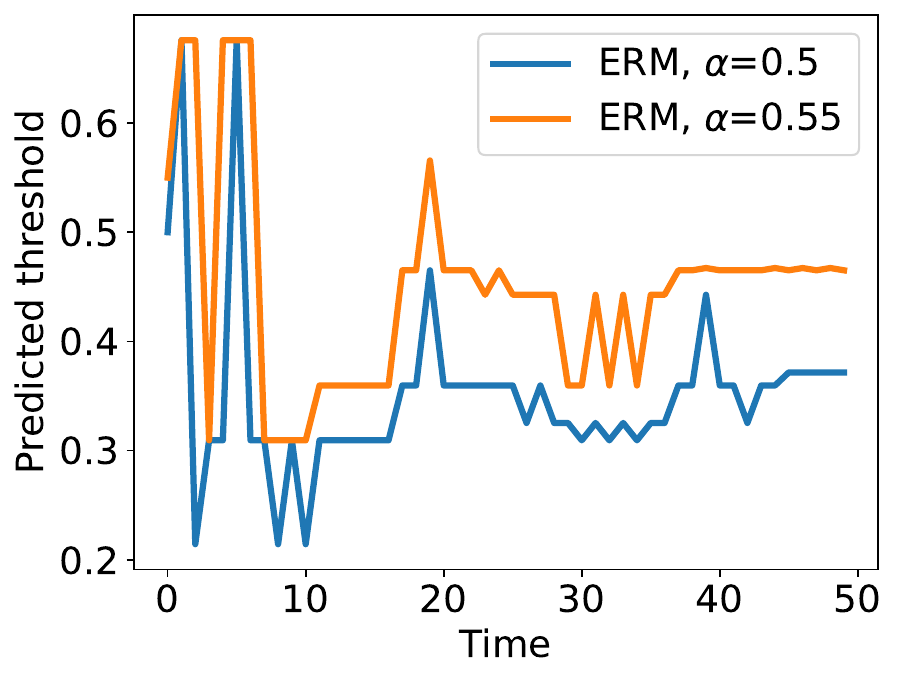}
    \end{minipage}
    \hfill
    \begin{minipage}{0.24\textwidth}
        \centering
        \includegraphics[width=\textwidth]{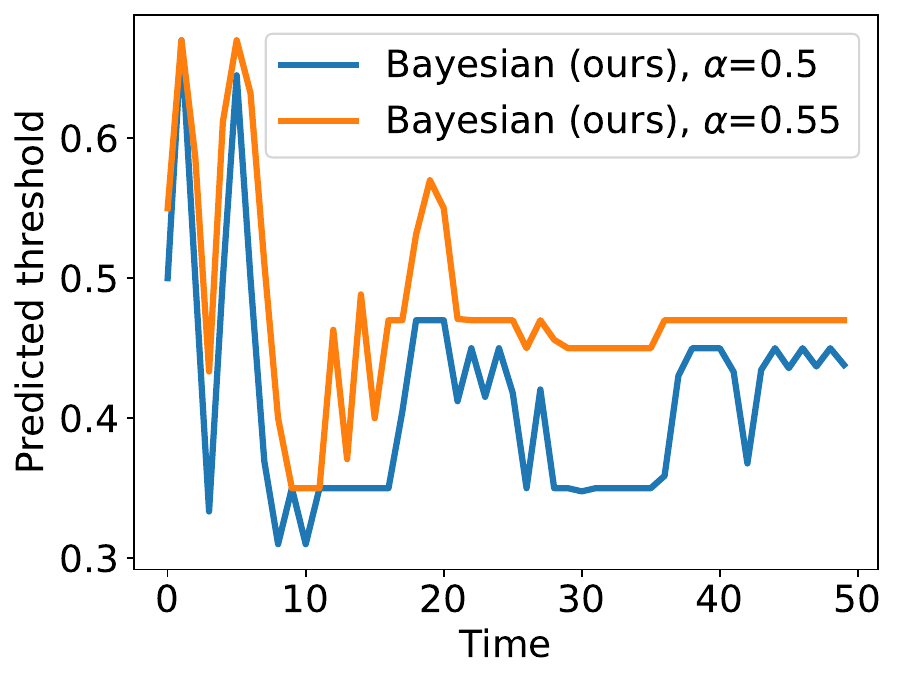}
    \end{minipage}
    \caption{Evaluating the monotonicity of threshold predictions. Ideally the orange line should be always above the blue line, since the associated target confidence level is higher. Columns correspond to different algorithms; rows correspond to different confidence level pairs.}
    \label{fig:monotonicity}
\end{figure}

The results are visualized in Figure~\ref{fig:monotonicity}. Ideally, in all the figures the orange line should be always above the blue line (i.e., the predicted confidence set due to Eq.(\ref{eq:cp}) is larger), since the associated confidence level $\alpha$ is higher. Unlike ERM and our algorithm, both MultiOGD and MultiMVP violate this property, which considerably weakens their trustworthiness to downstream users. We remark that although the data generation mechanism and the MultiMVP baseline are both randomized, a single random seed is used in this experiment to demonstrate the existence of the issue. 

\section{Algorithm and Analysis}\label{section:main}

In light of the above, our main contribution is Algorithm~\ref{alg:main}, a remarkably simple Bayesian analogue of the ERM approach. Let $P_0$ be an arbitrary distribution on $[0,R]$ called a \emph{prior}, with a strictly positive density function $p_0$. In each round, Algorithm~\ref{alg:main} maintains an algorithmic belief $P_t$ by mixing $P_0$ with the empirical distribution of the previous true scores, $\bar P(r^*_{1:t-1})$. Then, given each queried confidence level $\alpha$, it picks $r_t(\alpha)=q_{\alpha}(P_t)$ just like the ERM approach. We note that since $P_t$ does not depend on any specific $\alpha$, Algorithm~\ref{alg:main} can naturally support multiple arbitrary confidence level queries online. By construction, it is also clear that for any $\alpha_1<\alpha_2$ we always have $r_t(\alpha_1)\leq r_t(\alpha_2)$. That is, Algorithm~\ref{alg:main} does not suffer from the aforementioned monotonicity issue.

\begin{algorithm}[ht]
\caption{Online CP with regularized belief.\label{alg:main}}
\begin{algorithmic}[1]
\REQUIRE Step sizes $\{\lambda_t\}_{t\in\N_+}$, where $\lambda_1=1$, and $0<\lambda_t<1$ for all $t\geq 2$. A distribution $P_0$ on $[0:R]$ called a prior, with strictly positive density function $p_0$. 
\FOR{$t=1,2,\ldots$}
\STATE Compute the empirical distribution $\bar P (r^*_{1:t-1})$, and set the algorithmic belief $P_t$ to
\begin{equation}\label{eq:belief}
P_t=\lambda_tP_0+(1-\lambda_t)\bar P (r^*_{1:t-1}).
\end{equation}

\FOR{$\alpha\in A_t$}
\STATE Output the score threshold $r_t(\alpha)=q_{\alpha}(P_t)$.
\ENDFOR
\STATE Observe the true score $r^*_t$. 
\ENDFOR
\end{algorithmic}
\end{algorithm}

To proceed, the first two subsections respectively analyze Algorithm~\ref{alg:main} in the adversarial and the iid setting. They are followed by two extensions. 

\subsection{Analysis: Adversarial Setting}\label{subsection:adversarial}

Without any statistical assumption, we view Algorithm~\ref{alg:main} from the perspective of adversarial online learning. The key observation is that the mixture version of $P_t$ induces \emph{downstream regularizations} on the predicted threshold $r_t(\alpha)$, simultaneously for all $\alpha\in[0,1]$. Furthermore, with the uniform prior, the downstream regularizer becomes simply the quadratic function due to the structure of the quantile loss. 

Concretely, we define $\psi(r)\defeq \E_{r^*\sim P_0}[l_\alpha(r,r^*)]$. 
\begin{restatable}{theorem}{mainthm}\label{thm:main}
For all $\alpha\in[0,1]$, the output $r_t(\alpha)$ of Algorithm~\ref{alg:main} satisfies $r_1(\alpha)=\argmin_{r\in\R}\psi(r)$, and $\forall t\geq 2$,
\begin{equation}\label{eq:ftrl}
r_t(\alpha)=\argmin_{r\in\R}\spar{\frac{\lambda_t(t-1)}{1-\lambda_t}\psi(r)+\sum_{i=1}^{t-1}l_\alpha(r,r^*_i)}.
\end{equation}
Specifically,
\begin{itemize}
\item $\psi$ is strongly convex with coefficient $\inf_{r\in[0,R]}p_0(r)$, if the latter is positive.
\item If $P_0$ is the uniform distribution on $[0,R]$, then $\psi$ is a quadratic function centered at $\alpha R$,
\begin{equation*}
\psi(r)=\frac{1}{2R}r^2-\alpha r+\frac{1}{2}\alpha R.
\end{equation*}
\end{itemize}
\end{restatable}

Theorem~\ref{thm:main} shows that despite not knowing $\alpha$ at the beginning of the CP game, Algorithm~\ref{alg:main} generates the same output $r_t(\alpha)$ as a non-linearized \emph{Follow the Regularized Leader} (FTRL) algorithm on the quantile loss $l_\alpha$. Specifically, Eq.(\ref{eq:ftrl}) can be compared to the FTL-equivalence of the iid-based approach, Eq.(\ref{eq:erm}). The important difference is the additional regularizer $\psi(r)$. 

To provide more context here: FTRL is a standard improvement of ERM / FTL in adversarial online learning, with better stability and worst-case performance on ``difficult loss functions''. Our analysis involves the non-linearized version of FTRL, which has previously received less attention than its linearized counterpart. This is largely due to computational reasons, since non-linearized FTRL has to solve a convex optimization subroutine in each round, whereas linearized FTRL admits closed-form solutions \citep[Chapter~7.3]{orabona2023modern}. From this perspective, a notable novelty of our result is showing that for a class of benign regularizers, non-linearized FTRL on quantile losses can be simulated by a simple and efficient procedure. 

From Theorem~\ref{thm:main}, we can then obtain the regret bound of Algorithm~\ref{alg:main} using the standard FTRL analysis. In order to demonstrate the role of good priors, the strong convexity of the regularizer $\psi$ will be measured locally. 

\begin{restatable}{theorem}{regbound}\label{thm:regret}
Let $\mu_{t,\alpha}\defeq\inf\{p_0(r):r_t(\alpha)\wedge r^*_t\leq r\leq r_t(\alpha)\vee r^*_t\}$. With the step size $\lambda_t=1/\sqrt{t}$, Algorithm~\ref{alg:main} guarantees
\begin{equation}
\reg_T(\alpha)\defeq\sum_{t=1}^Tl_\alpha(r_t(\alpha),r^*_t)-\sum_{t=1}^Tl_\alpha(q_{\alpha}(r^*_{1:T}),r^*_t)=O\rpar{\psi(q_\alpha(r^*_{1:T}))\sqrt{T}+\sum_{t=1}^T\frac{1}{\mu_{t,\alpha}\sqrt{t}}},\label{eq:regret_full}
\end{equation}
for all time horizon $T$, true score sequence $r^*_{1:T}$, and confidence level $\alpha\in[0,1]$. Furthermore, if the prior $P_0$ is the uniform distribution on $[0,R]$, then
\begin{equation*}
\reg_T(\alpha)=O(R\sqrt{T}). 
\end{equation*}
\end{restatable}

Let us interpret this regret bound. Suppose the time horizon $T$ and the empirical true score distribution $\bar P(r^*_{1:T})$ are known beforehand (but the exact $r^*_{1:T}$ sequence is unknown), then for all $\alpha$, a very natural strategy is to predict $r_t(\alpha)=q_{\alpha}(r^*_{1:T})$. Theorem~\ref{thm:regret} shows that without any statistical assumption, Algorithm~\ref{alg:main} with the uniform prior asymptotically performs as well as this oracle in terms of the total quantile loss, and the $O(R\sqrt{T})$ regret bound is known to be tight \citep{hazan2023introduction,orabona2023modern}. Existing first-order optimization baselines are equipped with regret bounds of a similar type \citep{bhatnagar2023improved,gibbs2024conformal,zhang2024discounted}, but the difference is that they require knowing the confidence level $\alpha$ beforehand, whereas Algorithm~\ref{alg:main} achieves low regret simultaneously for all $\alpha\in[0,1]$.

\paragraph{The role of good prior} 
A strength of Theorem~\ref{thm:regret} is that the $O(R\sqrt{T})$ regret bound only requires $P_0$ being the uniform distribution. Nonetheless, 
if one has extra prior knowledge on the environment, picking a more sophisticated prior can indeed bring advantages. To see this, notice that the function $\psi$ in Eq.(\ref{eq:regret_full}) is minimized at $q_\alpha(P_0)$, therefore ideally we would aim for $P_0\approx \bar P(r^*_{1:T})$, which means $q_\alpha(P_0)\approx q_\alpha(r^*_{1:T})$ for all $\alpha$. But unlike the discrete distribution $\bar P(r^*_{1:T})$, $P_0$ also needs to have a ``positive enough'' density function, as otherwise the second term in Eq.(\ref{eq:regret_full}) would blow up.

\paragraph{Coverage frequency error} Existing first-order optimization baselines \citep{bhatnagar2023improved,gibbs2024conformal,zhang2024discounted} are equipped with both a regret bound and a coverage frequency error bound, Eq.(\ref{eq:coverage}). Hoping to challenge this convention, here we discuss the advantages of only considering the regret.

First, the coverage frequency error is fundamentally ``iterate-centric'', whereas an ideal performance metric needs to be ``data-centric''. Consider the CP interaction protocol displayed in Figure~\ref{fig:Diagram}: achieving low coverage frequency error requires the CP algorithm's output to depend not only on the top level (the nature and the base model), but also on the users' previous confidence level queries. This is in contrast with our regret minimization algorithm, whose output is independent of the users' query history. 

Furthermore, just like the pathological example given by \citep{bastani2022practical}, first-order optimization baselines essentially achieve the desirable coverage frequency due to the ``overshooting'' provided by the loss linearization. This is perhaps clear from the first online CP algorithm (ACI) proposed by \citep{gibbs2021adaptive}: regarding the update Eq.(\ref{eq:first_order_update}) with the constant learning rate $\eta_t=\eta$, it is shown that the coverage frequency error monotonically decreases as $\eta\rightarrow\infty$. Such a peculiar behavior results precisely from overshooting: if $\alpha=90\%$, then a failed coverage needs nine successful coverages to compensate, and \emph{ensuring} this does not have much to do with the observed data. This casts some natural doubt on the coverage frequency error that the algorithm is designed to optimize.

\paragraph{Relation to U-calibration} Our results so far deviate from the common scope of online learning, which requires specifying a single loss function in each round. Instead, they have a similar flavor as \emph{U-calibration} \citep{kleinberg2023u,luo2024optimal}: forecasting for an unknown downstream agent. Prior works on U-calibration considered the setting of \emph{finite-class distributional prediction} with generic proper losses, while our paper focuses on the continuous domain $[0,R]$ with the more specific quantile losses. The extra problem structure allows our algorithm to be deterministic (rather than being randomized like FTPL), thus establishing a closer connection to typical deterministic algorithms in \emph{online convex optimization}. 

\subsection{Analysis: IID Setting}

In practice, a CP algorithm is often applied without knowing the characteristics of the nature. We have been focusing on the adversarial setting, but what if the true scores $r^*_{1:T}$ turn out to be iid? We now demonstrate the \emph{adaptivity} of Algorithm~\ref{alg:main}: it automatically achieves almost the same guarantees as ERM under the additional iid assumption.

First, as the coverage probability becomes the default performance metric in the iid setting, we present the following bound on the \emph{dataset-conditional coverage probability}. Notice that the event of successful coverage can be expressed as $r^*_t\leq r_t(\alpha)$, where $r_t(\alpha)$ is determined by the past true scores $r^*_{1:t-1}$ and the queried $\alpha$. 

\begin{restatable}{theorem}{coverage}\label{thm:coverage}
Assume the true score sequence $r^*_1,r^*_2,\ldots$ is drawn iid from an unknown continuous distribution $\mathcal{D}$. With the step size $\lambda_t=1/\sqrt{t}$ and an arbitrary prior $P_0$, Algorithm~\ref{alg:main} guarantees that for any fixed $t\geq 2$, with probability at least $1-\delta$ over the randomness of $r^*_{1:t-1}$, we have for all $\alpha\in[0,1]$, 
\begin{equation*}
\alpha-\sqrt{\frac{\log(2/\delta)}{2(t-1)}}-\frac{1}{\sqrt{t}-1}\leq\P_{r^*_t\sim\mathcal{D}}\spar{r^*_t\leq r_t(\alpha)}
\leq \alpha+\sqrt{\frac{\log(2/\delta)}{2(t-1)}}+\frac{1}{\sqrt{t}-1}+\frac{1}{t-1}.
\end{equation*}
\end{restatable}

Compared to the analogous result for ERM \citep[Theorem~34]{roth2022uncertain}, the difference here due to the regularization is the $(\sqrt{t}-1)^{-1}$ factor, which is dominated by the existing $O(\sqrt{t^{-1}\log\delta^{-1}})$ term resulting from randomness. That is, Algorithm~\ref{alg:main} achieves almost the same dataset-conditional coverage probability error as Split Conformal, despite being agnostic to the iid assumption. 

Besides the coverage probability, we can also analyze the \emph{excess quantile risk} of Algorithm~\ref{alg:main}, which matches the standard oracle inequality one would obtain using ERM. 

\begin{restatable}{theorem}{risk}\label{thm:risk}
Assume the true score sequence $r^*_1,r^*_2,\ldots$ is drawn iid from an unknown distribution $\mathcal{D}$. With the step size $\lambda_t=1/\sqrt{t}$ and an arbitrary prior $P_0$, Algorithm~\ref{alg:main} guarantees that for any fixed $t\geq 2$, with probability at least $1-\delta$ over the randomness of $r^*_{1:t-1}$, we have for all $\alpha\in[0,1]$, 
\begin{equation*}
\E_{r^*_t\sim\mathcal{D}}[l_\alpha(r_t(\alpha),r^*_t)]
\leq
\min_{r\in[0,R]}\E_{r^*_t\sim\mathcal{D}}[l_\alpha(r,r^*_t)]+O\rpar{R\sqrt{\frac{\log(1/\delta)}{t}}}.
\end{equation*}
\end{restatable}

\paragraph{Bayesian interpretation} We have been calling Algorithm~\ref{alg:main} ``Bayesian''. Following \citep[Chapter~23]{gelman2021bayesian}, we now make this important connection concrete. 

Consider the following distribution estimation problem: given $x_1,\ldots,x_n\in\calX$ sampled iid from an unknown distribution $X$, what is a good estimate of $X$? As opposed to the frequentist estimate $\bar P(x_{1:n})$, a Bayesian estimator would place a prior $F_0$ over all distributions supported on the domain $\calX$, compute the posterior $F_n$ from the samples, and output the mean $\E[F_n]$. 

For analytical convenience, one would typically choose $F_0$ as a \emph{conjugate prior}: it refers to a family of priors such that if $F_0$ belongs to this family, then $F_n$ also belongs to this family. The most notable conjugate prior for distribution estimation is the \emph{Dirichlet process} (DP), denoted as $\mathrm{DP}(\alpha,P_0)$. Here $\alpha$ and $P_0$ are hyperparameters: $P_0$ equals the mean $\E[\mathrm{DP}(\alpha,P_0)]$, while $\alpha$ controls the variance of $\mathrm{DP}(\alpha,P_0)$. Due to the conjugacy, if $F_0=\mathrm{DP}(\alpha,P_0)$, then
\begin{equation*}
F_n=\mathrm{DP}\rpar{\alpha+n,\frac{\alpha}{\alpha+n}P_0+\frac{n}{\alpha+n}\bar P(x_{1:n})}.
\end{equation*}
Consequently, the Bayesian estimator of the distribution $X$ is
\begin{equation*}
\E[F_n]=\frac{\alpha}{\alpha+n}P_0+\frac{n}{\alpha+n}\bar P(x_{1:n}).
\end{equation*}
This is the same as the belief update Eq.(\ref{eq:belief}) in Algorithm~\ref{alg:main}, with the hyperparameter $\lambda_t=\alpha/(\alpha+n)$. Our work can thus be regarded as the analysis of a Bayesian algorithm \emph{with and without} the traditional iid assumption. 

Back to the adversarial setting, we conclude this paper with two extensions of Algorithm~\ref{alg:main}. 

\subsection{Extension: Quantization}\label{subsection:quantized}

Recall our construction of MultiOGD from Section~\ref{section:validity}. Although not studied by existing works, it is not hard to see that with the size of the grid $\tilde A$ being $O(\sqrt{T})$, MultiOGD also satisfies the same $\alpha$-agnostic $O(R\sqrt{T})$ regret bound as in Theorem~\ref{thm:regret}, since the quantile loss $l_\alpha(r,r^*)$ is $R$-Lipschitz with respect to $\alpha$. This raises a natural question: Algorithm~\ref{alg:main} requires $O(T)$ memory due to storing the empirical distribution of previous true scores -- can we reduce it to $O(\sqrt{T})$? 

\paragraph{Quantized algorithm} Here is a variant of Algorithm~\ref{alg:main}, denoted as \textsc{Quantized}, achieving this goal. The idea is to discretize the domain $[0,R]$ rather than the $\alpha$-space: we maintain an evenly-spaced grid of size $\sqrt{T}$ over $[0,R]$, round each observed $r^*_t$ to its nearest neighbor $\tilde r^*_t$ on the grid, and replace the belief update Eq.(\ref{eq:belief}) by
\begin{equation*}
P_t=\lambda_tP_0+(1-\lambda_t)\bar P (\tilde r^*_{1:t-1}).
\end{equation*}

The associated regret bound follows from the Lipschitzness of $l_\alpha(r,r^*)$ with respect to $r^*$. 
\begin{restatable}{theorem}{quantized}\label{thm:quantized}
With $\lambda_t=1/\sqrt{t}$ and the uniform $P_0$, \textsc{Quantized} achieves $\reg_T(\alpha)=O(R\sqrt{T})$.
\end{restatable}

Compared to MultiOGD, \textsc{Quantized} achieves the same $O(R\sqrt{T})$ regret bound with $O(\sqrt{T})$ memory, while avoiding its monotonicity issue. There is another practical advantage: after observing each $r^*_t$, MultiOGD needs to update all $\sqrt{T}$ base algorithms, whereas \textsc{Quantized} performs only one update on the algorithmic belief $\bar P_t$, and then makes $\abs{A_t}$ inferences using the prediction head. 

\subsection{Extension: Continual Distribution Shift}

Starting from \citep{gibbs2021adaptive}, the study of adversarial online CP has been largely motivated by the prevalence of continual distribution shifts in practice. Tackling this challenge requires \emph{non-converging} algorithms characterized by sliding-window performance guarantees. We finally present a discounted variant of Algorithm~\ref{alg:main}, denoted by \textsc{Discounted}, along this direction. 

\paragraph{Discounted algorithm} Let $\beta\in(0,1)$ be a \emph{discount factor}, which is a bandwidth hyperparameter required by \textsc{Discounted}. Then, we define a regularized and discounted empirical distribution of $r^*_{1:t}$ recursively by
\begin{equation*}
\bar P_\beta(r^*_1)=\beta P_0+(1-\beta)\delta(r^*_1),
\end{equation*}
\begin{equation*}
\bar P_\beta(r^*_{1:t})=\beta\bar P_\beta(r^*_{1:t-1})+(1-\beta)\delta(r^*_t)=\beta^{t}P_0+(1-\beta)\sum_{i=1}^t\beta^{t-i}\delta(r^*_i),
\end{equation*}
where $\delta(r^*_t)$ is the distribution with point mass at $r^*_t$. This is used to replace the undiscounted empirical distribution in the belief update, i.e., Eq.(\ref{eq:belief}) is replaced by
\begin{equation*}
P_t=\lambda_tP_0+(1-\lambda_t)\bar P_\beta (r^*_{1:t-1}).
\end{equation*}
After that, the prediction head remains unchanged, i.e., $r_t(\alpha)=q_\alpha(P_t)$. 

Similar to Theorem~\ref{thm:main} and \ref{thm:regret}, we can show that \textsc{Discounted} simulates the $\beta$-discounted non-linearized FTRL, which is equipped with a $\beta$-discounted regret bound. Importantly, reasonable step sizes $\lambda_t$ become constant (rather than decreasing), which emphasizes the crucial role of the prior $P_0$: instead of only using $P_0$ to regularize the beginning of the CP game, \textsc{Discounted} continually mix $P_0$ into its algorithm belief with constant weight, such that it does not ``overfit the current environment''. 

\begin{restatable}{theorem}{discounted}\label{thm:discounted}
With $\lambda_t=\lambda=\frac{\sqrt{1-\beta}}{\beta+\sqrt{1-\beta}}$ and the uniform $P_0$, the output $r_t(\alpha)$ of \textsc{Discounted} satisfies
\begin{equation*}
r_t(\alpha)=\argmin_{r\in\R}\spar{(1-\beta)^{-1}\rpar{\frac{\lambda}{1-\lambda}+\beta^{t-1}}\psi(r)
+\sum_{i=1}^{t-1}\beta^{t-1-i}l_\alpha(r,r^*_i)},
\end{equation*}
for all $\alpha$ and $t$. In addition, for all $\alpha\in[0,1]$, it guarantees the discounted regret bound
\begin{equation*}
\reg_{T,\beta}(\alpha)
\defeq \sum_{t=1}^T\beta^{T-t}l_\alpha(r_t(\alpha),r_t^*)-\min_{r\in[0,R]}\sum_{t=1}^T\beta^{T-t}l_\alpha(r,r_t^*)
\leq\frac{R}{\sqrt{1-\beta}}+o(R),
\end{equation*}
where $o(\cdot)$ is with respect to $T\rightarrow \infty$.
\end{restatable}

We remark that \citep[Theorem~7]{zhang2024discounted} presents a discounted regret lower bound on linear losses, which can be converted to $\Omega(\min\{\alpha,1-\alpha\}R/\sqrt{1-\beta^2})$ on the quantile losses we consider. Since $(1-\beta)^{-1/2}\leq 2(1-\beta^2)^{-1/2}$ for all $\beta\in(0,1)$, Theorem~\ref{thm:discounted} matches this lower bound in the minimax sense (with respect to $\alpha$, i.e., when $\alpha=1/2$). 

\section{Experiment}\label{section:experiment}

Complementing our theoretical results, we now evaluate the performance of our Bayesian approach using more experiments. Code is available at \url{https://github.com/zhiyuzz/Bayesian-Conformal}.

\begin{figure}[ht]
\centering
    \begin{minipage}{0.369\textwidth}
    \includegraphics[width=\textwidth]{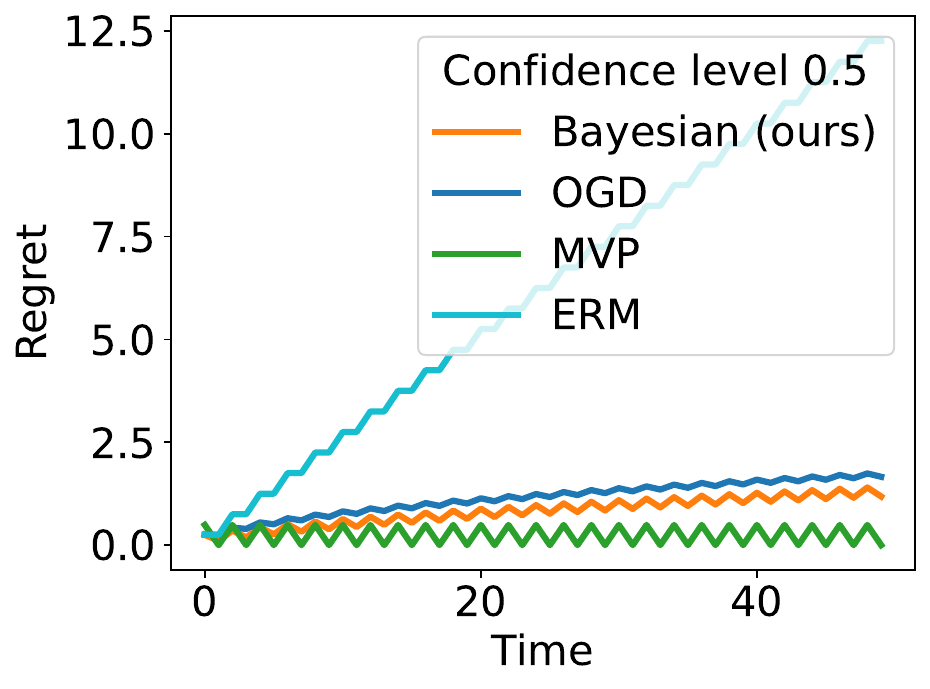}
    \end{minipage}
    \begin{minipage}{0.35\textwidth}
        \includegraphics[width=\textwidth]{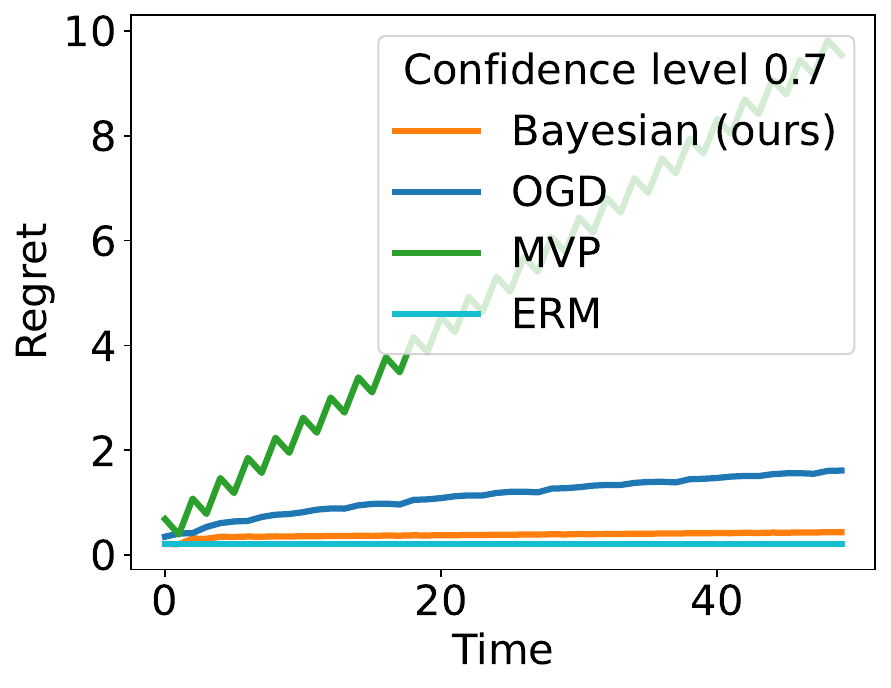}
    \end{minipage}
    \caption{Regret on switching data.}
    \label{fig:switching}
\end{figure}

\paragraph{Switching sequence} First, to demonstrate the failure of ERM without the iid assumption, we consider a synthetic $r^*_{1:T}$ sequence which switches in every round between 0 and 1. Similar to Section~\ref{section:validity}, four algorithms are tested: OGD \citep{gibbs2021adaptive}, MVP \citep{bastani2022practical},\footnote{Similar to Section~\ref{section:validity}, we use a single random seed for the MVP baseline throughout this section, since we find the results to be generally insensitive to the seed.} ERM and our Bayesian algorithm \textsc{Quantized}. Figure~\ref{fig:switching} plots their regret measured by the quantile loss, under two different $\alpha$ values. 

Consistent with the classical online learning theory, ERM becomes brittle when $\alpha$ matches the long run average of $r^*_{1:T}$ (i.e, $0.5$), suffering linear regret with respect to $T$. In contrast, both OGD (with $\eta_t=t^{-1/2}$; $\alpha$ is known) and our Bayesian algorithm achieve sublinear regret under both $\alpha$ values. Quite different from the conventional online learning framework, MVP is designed to minimize the conditional empirical coverage error, but nonetheless, it achieves low regret when $\alpha=0.5$. The limitation is that MVP requires a relatively long period to warm up: when $\alpha=0.7$, the regret of MVP grows linearly at the beginning, before hitting a plateau at $T\approx 800$. 

\paragraph{Stock price} Next, we move to an actual CP problem: predicting the time-varying volatility of the stock price, with the base model being a standard time series forecasting method called GARCH \citep{bollerslev1986generalized}. This experiment was designed by \cite{gibbs2021adaptive} and further studied by \cite{bastani2022practical}. See \citep[Appendix~B.3.1]{bastani2022practical} for the specifics of its context. 

Two baselines are considered: a specialization of OGD (ACI) for time series forecasting, and MVP. Besides requiring a fixed learning rate, the former operates on a sliding time window whose length is also a hyperparameter. Similarly, MVP requires picking the size of discretization. For both baselines, we follow the exact implementation from \citep{bastani2022practical}, including the hyperparameters. 

As for our Bayesian approach, we adopt the discounted version to handle the continual distribution shift, together with quantization. The discretization grid $\tilde A$ has the same size as the MVP baseline, and we pick the discount factor $\beta$ such that the effective length $(1-\beta^2)^{-1}$ of the discounted time window exactly matches the length of the ACI baseline's sliding window. Given this $\beta$, $\lambda_t$ is selected according to Theorem~\ref{thm:discounted}. It means that compared to the baselines, our algorithm cannot benefit from any extra hyperparameter tuning. 

\begin{figure}[ht]
    \centering
\includegraphics[width=0.9\linewidth]{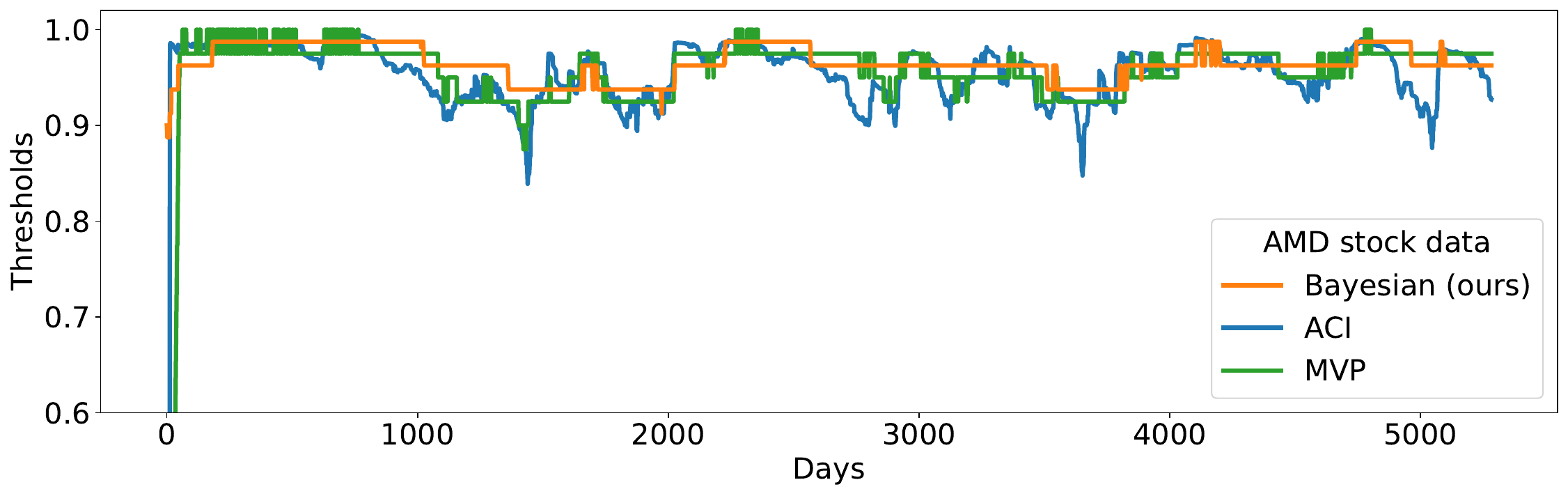}
    \caption{Predicted score threshold on AMD stock data.}
    \label{fig:stock_threshold_amd}
\end{figure}

\begin{figure}[ht]
\centering
    \begin{minipage}{0.35\textwidth}
        \includegraphics[width=\textwidth]{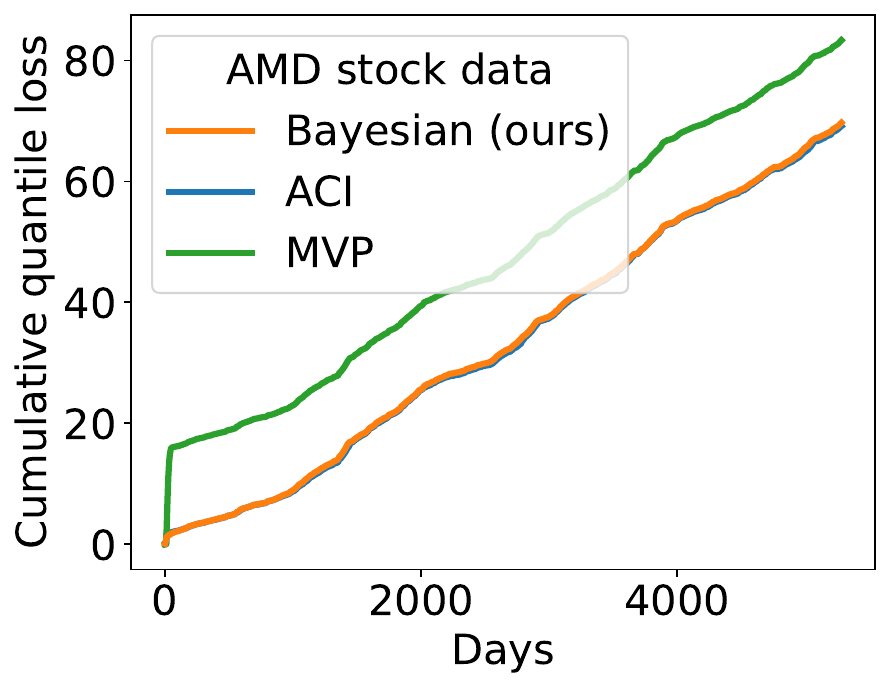}
    \end{minipage}
    \begin{minipage}{0.35\textwidth}
        \includegraphics[width=\textwidth]{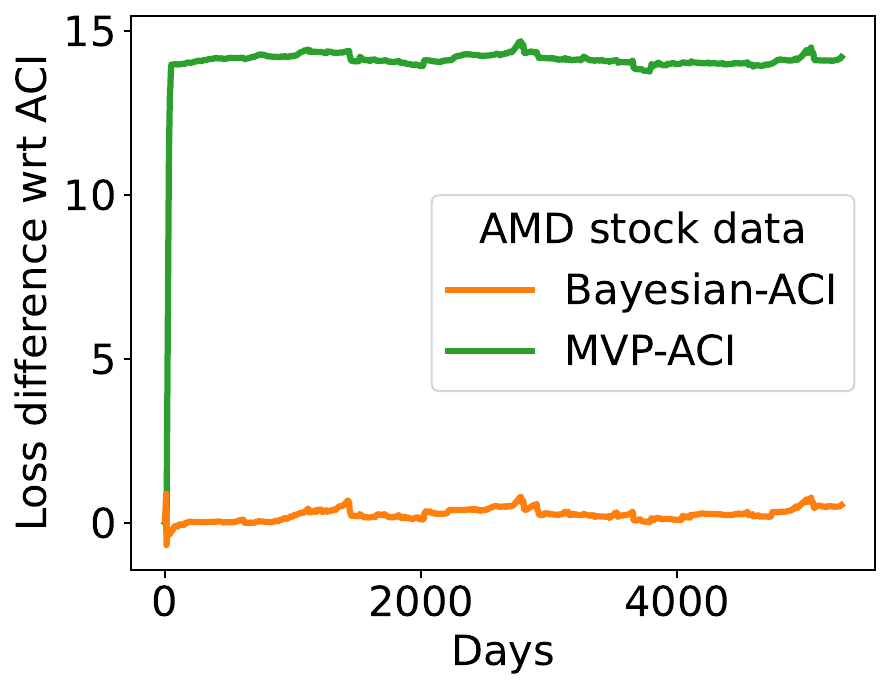}
    \end{minipage}
    \caption{Quantile loss on AMD stock data.}
    \label{fig:loss_amd}
\end{figure}

With $\alpha=0.9$, Figure~\ref{fig:stock_threshold_amd} plots the $r_{1:T}(\alpha)$ sequence predicted by different algorithms. As a visual sanity check, our algorithm generates a reasonable prediction sequence with slightly less fluctuation than the baselines. To make a more concrete comparison, Figure~\ref{fig:loss_amd} plots the total quantile loss suffered by all three algorithms, as well as the difference compared to ACI. It shows that our algorithm achieves almost the same total loss as ACI, and it is faster to warm up than MVP. 

Finally, we also evaluate the empirical coverage rate of the tested algorithms. Although our algorithm is not designed for this metric, it performs competitively compared to the baselines. The target is $1-\alpha=0.9$, and closer to this target is better. ACI achieves $0.901$, MVP achieves $0.893$, and our Bayesian algorithm achieves $0.899$. Appendix~\ref{section:additional_experiment} includes results on a different stock dataset.

\section{Conclusion}

Focusing on the online adversarial formulation of conformal prediction, this paper demonstrates various benefits of being Bayesian. Specifically, we propose a novel algorithm with a number of strengths -- it supports multiple arbitrary confidence level queries, achieves probably low regret, avoids a previously overlooked monotonicity issue on the obtained confidence sets, and adapts to iid environments. We further develop quantized and discounted extensions of this algorithm, and our theoretical arguments are supported by experiments. 

For future directions, we believe that an important remaining problem in online conformal prediction is to rigorously characterize the strengths and limitations of various performance metrics. This paper contains an argument against the coverage frequency error, but additional effort is needed to make it more quantitative and concrete. Besides, it is valuable to have \emph{group-conditional} guarantees analogous to \cite{bastani2022practical,noarov2023high}, which is an interesting problem left for future works. 

\section*{Acknowledgement}

This work is partially funded by Harvard University Dean’s Competitive Fund for Promising Scholarship.

\bibliography{Bayesian_CP}

\newpage
\section*{Appendix}
\appendix

\section{Omitted Proofs}\label{section:proofs}

\mainthm*

\begin{proof}[Proof of Theorem~\ref{thm:main}]
We first rewrite the base regularizer $\psi$ as
\begin{align*}
\psi(r)&=\int_0^Rl_\alpha(r,r^*)p_0(r^*)dr^*\\
&=(1-\alpha)\int_0^r(r-r^*)p_0(r^*)dr^*+\alpha\int_r^R(r^*-r)p_0(r^*)dr^*.
\end{align*}
It is twice-differentiable, with
\begin{equation*}
\psi'(r)=(1-\alpha)\int_0^rp_0(r^*)dr^*-\alpha\int_r^Rp_0(r^*)dr^*=\int_0^rp_0(r^*)dr^* -\alpha,
\end{equation*}
and $\psi''(r)=p_0(r)$. The strong convexity statement on $\psi$ is thus clear. If $P_0$ is uniform, we have
\begin{align*}
\psi(r)&=R^{-1}\spar{(1-\alpha)\int_0^r(r-r^*)dr^*+\alpha\int_r^R(r^*-r)dr^*}\\
&=\frac{1}{2R}\spar{(1-\alpha) r^2+\alpha(R-r)^2}=\frac{1}{2R}r^2-\alpha r+\frac{1}{2}\alpha R.
\end{align*}

Next, consider the first part of the theorem. The case of $t=1$ is straightforward to verify. For any $t\geq 2$, Algorithm~\ref{alg:main} outputs
\begin{align}
r_t(\alpha)&=q_{\alpha}\spar{\lambda_tP_0+(1-\lambda_t)\bar P (r^*_{1:t-1})}\nonumber\\
&=\min\left\{r:\lambda_t\int_0^rp_0(r^*)dr^*+\frac{1-\lambda_t}{t-1}\sum_{i=1}^{t-1}\bm{1}[r^*_i\leq r]\geq \alpha\right\}. \label{eq:compare_to}
\end{align}
On the other hand, consider the optimization objective in Eq.(\ref{eq:ftrl}), which we write as
\begin{equation}\label{eq:ftrl_objective}
F_t(r)\defeq \frac{\lambda_t(t-1)}{1-\lambda_t}\psi(r)+\sum_{i=1}^{t-1}l_\alpha(r,r^*_i).
\end{equation}
Notice that the function $F_t(r)$ is continuous and right-differentiable. Taking its right-derivative, we have
\begin{align*}
F'_{t,+}(r)&=\frac{\lambda_t(t-1)}{1-\lambda_t}\spar{\int_0^rp_0(r^*)dr^* -\alpha}+\rpar{-\alpha\sum_{i=1}^{t-1}\bm{1}[r< r^*_i]+(1-\alpha)\sum_{i=1}^{t-1}\bm{1}[r\geq r^*_i]}\\
&=\frac{\lambda_t(t-1)}{1-\lambda_t}\int_0^rp_0(r^*)dr^*-\frac{\alpha \lambda_t(t-1)}{1-\lambda_t}-\alpha(t-1)+\sum_{i=1}^{t-1}\bm{1}[r\geq r^*_i]\\
&=\frac{t-1}{1-\lambda_t}\rpar{\lambda_t\int_0^rp_0(r^*)dr^*+\frac{1-\lambda_t}{t-1}\sum_{i=1}^{t-1}\bm{1}[r\geq r^*_i]-\alpha}.
\end{align*}
Comparing it to Eq.(\ref{eq:compare_to}), we see that the output $r_t(\alpha)$ of Algorithm~\ref{alg:main}, given by Eq.(\ref{eq:compare_to}), satisfies
\begin{equation*}
r_t(\alpha)=\min\{r:F'_{t,+}(r)\geq 0\}. 
\end{equation*}
Since the function $F_{t}(r)$ is strictly convex, we have $r_t(\alpha)=\argmin_r F_t(r)$, which is equivalent to Eq.(\ref{eq:ftrl}).
\end{proof}

\regbound*

\begin{proof}[Proof of Theorem~\ref{thm:regret}]
The proof can be decomposed into the following steps.

\paragraph{Step 1} Starting from the FTRL formulation Eq.(\ref{eq:ftrl}), we first verify that the regularizer weight $\frac{\lambda_t(t-1)}{1-\lambda_t}$ is increasing with respect to $t$ (when $t>1$), which is required by the FTRL analysis. To this end, define
\begin{equation*}
h_t\defeq \frac{\lambda_t(t-1)}{1-\lambda_t}=\frac{t-1}{\sqrt{t}-1}. 
\end{equation*}
Taking the derivative with respect to $t$, for all $t>1$,
\begin{equation*}
\frac{dh_t}{dt}=\frac{\sqrt{t}-1-\frac{t-1}{2\sqrt{t}}}{(\sqrt{t}-1)^2}=\frac{t-2\sqrt{t}+1}{2\sqrt{t}(\sqrt{t+1}-1)^2}=\frac{(\sqrt{t}-1)^2}{2\sqrt{t}(\sqrt{t+1}-1)^2}\geq 0. 
\end{equation*}
For completeness, we also define $h_1=1$. 

Besides, we have the order estimate $h_t=O(\sqrt{t})$, $1/h_t=O(1/\sqrt{t})$, where $O(\cdot)$ only hides an absolute constant. 

\paragraph{Step 2} Next, due to Theorem~\ref{thm:main}, we can apply the standard FTRL analysis. Recall our notation from Eq.(\ref{eq:ftrl_objective}): we write the optimization objective in Eq.(\ref{eq:ftrl}) as
\begin{equation*}
F_t(r)\defeq h_t\psi(r)+\sum_{i=1}^{t-1}l_\alpha(r,r^*_i), \quad\forall t\geq 2.
\end{equation*}
Similarly, we also write $F_1(r)\defeq h_1\psi(r)$. Notice that $r_t(\alpha)=\argmin_{r\in\R} F_t(r)$ for all $t$. 

The classical FTRL equality lemma \citep[Lemma~7.1]{orabona2023modern} states that
\begin{align*}
\reg_T(\alpha)&=h_{T+1}\psi(q_\alpha(r^*_{1:T}))-\min_{r\in\R}\psi(r)+\sum_{t=1}^T\spar{F_t(r_t(\alpha))-F_{t+1}(r_{t+1}(\alpha))+l_\alpha(r_t(\alpha),r^*_t)}\\
&\qquad+F_{T+1}(r_{T+1}(\alpha))-F_{T+1}(q_\alpha(r^*_{1:T}))\\
&\leq h_{T+1}\psi(q_\alpha(r^*_{1:T}))+\sum_{t=1}^T\spar{F_t(r_t(\alpha))-F_{t+1}(r_{t+1}(\alpha))+l_\alpha(r_t(\alpha),r^*_t)},
\end{align*}
where the second line is due to $\min_r\psi(r)\geq 0$, and $r_{T+1}(\alpha)=\argmin_{r\in\R}F_{T+1}(r)$. 

Consider the sum on the RHS, where for conciseness we omit $(\alpha)$ in the notation. This is the typical one-step quantity involved in the FTRL analysis. Following a similar procedure as \citep[Lemma~7.8]{orabona2023modern}, we have
\begin{align*}
&F_t(r_t)-F_{t+1}(r_{t+1})+l_\alpha(r_t,r^*_t)\\
=~&F_t(r_t)+l_\alpha(r_t,r^*_t)-F_{t}(r_{t+1})-l_\alpha(r_{t+1},r^*_t)+(h_t-h_{t+1})\psi(r_{t+1})\\
\leq~& F_t(r_t)+l_\alpha(r_t,r^*_t)-F_{t}(r_{t+1})-l_\alpha(r_{t+1},r^*_t)\tag{$h_{t+1}\geq h_t$, $\psi(r_{t+1})\geq 0$}\\
\leq~& F_t(r_t)+l_\alpha(r_t,r^*_t)-\min_{r\in\R}\spar{F_{t}(r)+l_\alpha(r,r^*_t)}.
\end{align*}
Observe that since $F_t(\cdot)$ and $l_\alpha(\cdot,r^*_t)$ are both convex, the minimizing argument of their sum lies between their respective unique minimizers, $r_t$ and $r^*_t$. On this segment, the function $F_t$ is $h_t\mu_{t,\alpha}$-strongly-convex, where $\mu_{t,\alpha}$ is defined in the assumption of the theorem. We now proceed using the property of strong convexity \citep[Lemma~7.6]{orabona2023modern}, which we restate as Lemma~\ref{lemma:strong_convexity}. 

Concretely, if $g_t$ is a subgradient of $l_\alpha(\cdot,r^*_t)$ at $r_t$, then it is also a subgradient of $F_t(\cdot)+l_\alpha(\cdot,r^*_t)$ at $r_t$, since $r_t=\argmin_r F_t(r)$. Moreover, such a subgradient $g_t$ satisfies $\abs{g_t}\leq 1$ due to $l_\alpha(\cdot,r^*_t)$ being $1$-Lipschitz. Combining these with the strong convexity, Lemma~\ref{lemma:strong_convexity} yields
\begin{equation*}
F_t(r_t)+l_\alpha(r_t,r^*_t)-\min_{r\in\R}\spar{F_{t}(r)+l_\alpha(r,r^*_t)}\leq \frac{1}{2h_t\mu_{t,\alpha}}. 
\end{equation*}
Plugging this all the way back into the regret bound, we have
\begin{align*}
\reg_T(\alpha)&\leq h_{T+1}\psi(q_\alpha(r^*_{1:T}))+\frac{1}{2}\sum_{t=1}^T\frac{1}{h_t\mu_{t,\alpha}}\\
&=O\rpar{\psi(q_\alpha(r^*_{1:T}))\sqrt{T}+\sum_{t=1}^T\frac{1}{\mu_{t,\alpha}\sqrt{t}}}.
\end{align*}

\paragraph{Step 3} Finally we analyze the special case of uniform prior. From Theorem~\ref{thm:main},
\begin{equation*}
\psi(q_\alpha(r^*_{1:T}))\leq \max_{r\in[0,R]}\psi(r)\leq \max_{r\in[0,R]}\rpar{\frac{1}{2R}r^2-\alpha r+\frac{1}{2}\alpha R}\leq \frac{R}{2}.
\end{equation*}
Furthermore, $\mu_{t,\alpha}=1/R$. Plugging in $\sum_{t=1}^Tt^{-1/2}=O(\sqrt{T})$ completes the proof.
\end{proof}

\coverage*

\begin{proof}[Proof of Theorem~\ref{thm:coverage}]
The proof follows a similar strategy as \citep[Theorem~34]{roth2022uncertain}. First, for any fixed $t\geq 2$, the samples $r^*_{1:t-1}$ have no ties almost surely, since the underlying distribution $\mathcal{D}$ is continuous. We will condition the rest of the analysis on this event. 

Next, recall Algorithm~\ref{alg:main}'s prediction rule, Eq.(\ref{eq:compare_to}). On one hand, we have
\begin{equation*}
\lambda_t\int_0^{r_t(\alpha)}p_0(r^*)dr^*+\frac{1-\lambda_t}{t-1}\sum_{i=1}^{t-1}\bm{1}[r^*_i\leq r_t(\alpha)]\geq \alpha,
\end{equation*}
which means
\begin{equation*}
\frac{1-\lambda_t}{t-1}\sum_{i=1}^{t-1}\bm{1}[r^*_i\leq r_t(\alpha)]\geq \alpha-\lambda_t,
\end{equation*}
\begin{equation*}
\frac{1}{t-1}\sum_{i=1}^{t-1}\bm{1}[r^*_i\leq r_t(\alpha)]\geq \alpha+\frac{\lambda_t}{1-\lambda_t}(\alpha-1)\geq\alpha-\frac{1}{\sqrt{t}-1}.
\end{equation*}

On the other hand, if we define $m=\sum_{i=1}^{t-1}\bm{1}[r^*_i\leq r_t(\alpha)]$ and let $r^*_{-1}$ be the $(m-1)$-th smallest element of $r^*_{1:t-1}$, then it is also clear from Eq.(\ref{eq:compare_to}) that
\begin{equation*}
\lambda_t\int_0^{r^*_{-1}}p_0(r^*)dr^*+\frac{1-\lambda_t}{t-1}\sum_{i=1}^{t-1}\bm{1}[r^*_i\leq r^*_{-1}]\leq \alpha,
\end{equation*}
which means
\begin{equation*}
\frac{1-\lambda_t}{t-1}\sum_{i=1}^{t-1}\bm{1}[r^*_i\leq r^*_{-1}]\leq \alpha-\lambda_t\int_0^{r^*_{-1}}p_0(r^*)dr^*\leq \alpha,
\end{equation*}
\begin{equation*}
\frac{1}{t-1}\sum_{i=1}^{t-1}\bm{1}[r^*_i\leq r^*_{-1}]\leq \frac{\alpha}{1-\lambda_t}\leq \alpha+\frac{1}{\sqrt{t}-1},
\end{equation*}
\begin{equation*}
\frac{1}{t-1}\sum_{i=1}^{t-1}\bm{1}[r^*_i\leq r_t(\alpha)]\leq \frac{1}{t-1}\sum_{i=1}^{t-1}\bm{1}[r^*_i\leq r^*_{-1}]+\frac{1}{t-1}\leq \alpha+\frac{1}{\sqrt{t}-1}+\frac{1}{t-1}.
\end{equation*}

In summary, 
\begin{equation}\label{eq:sample_quantile}
\alpha-\frac{1}{\sqrt{t}-1}\leq \frac{1}{t-1}\sum_{i=1}^{t-1}\bm{1}[r^*_i\leq r_t(\alpha)]\leq \alpha+\frac{1}{\sqrt{t}-1}+\frac{1}{t-1}.
\end{equation}

Finally we apply the DKW inequality (Lemma~\ref{lemma:dkw}). For all $\eps>0$, we have
\begin{equation*}
\P_{r^*_{1:t-1}}\spar{\sup_{\alpha\in[0,1]}\abs{\rpar{\frac{1}{t-1}\sum_{i=1}^{t-1}\bm{1}[r^*_i\leq r_t(\alpha)]}-\P_{r^*_t}[r^*_t\leq r_t(\alpha)]}>\eps}\leq 2\exp\spar{-2(t-1)\eps^2}.
\end{equation*}
Therefore, with probability at least $1-\delta$ over the randomness of $r^*_{1:t-1}$, we have
\begin{equation*}
\abs{\rpar{\frac{1}{t-1}\sum_{i=1}^{t-1}\bm{1}[r^*_i\leq r_t(\alpha)]}-\P_{r^*_t}[r^*_t\leq r_t(\alpha)]}\leq \sqrt{\frac{\log(2/\delta)}{2(t-1)}},\quad\forall\alpha\in[0,1].
\end{equation*}
Combining it with Eq.(\ref{eq:sample_quantile}) above completes the proof. 
\end{proof}

\risk*

\begin{proof}[Proof of Theorem~\ref{thm:risk}]
The proof follows from a standard uniform convergence argument \citep{zhang2023mathematical} combined with the Lipschitzness of the quantile loss. 

First, notice that with any combination of $\alpha$, $r$ and $r^*$, the quantile loss $l_\alpha(r,r^*)\in[0,R]$. Therefore, fixing any $\alpha\in[0,1]$ and $r\in[0,R]$, we apply the Hoeffding's inequality (Lemma~\ref{lemma:hoeffding}) to obtain
\begin{equation*}
\P_{r^*_{1:t-1}}\spar{\abs{\frac{1}{t-1}\sum_{i=1}^{t-1}l_\alpha(r,r^*_i)-\E_{r^*_t}[l_\alpha(r,r^*_t)]}\geq \eps}\leq 2\exp\rpar{-\frac{2(t-1)\eps^2}{R^2}}.
\end{equation*}

Next, we evenly discretize $[0,1]$ by a grid of size $\sqrt{t}$, and also $[0,R]$ by a grid of size $\sqrt{t}$, and denote their combination as a set $S$. $\abs{S}=t$. For all $\alpha$ and $r$, there exists $(\tilde \alpha,\tilde r)\in S$ satisfying $\abs{\alpha-\tilde\alpha}\leq 1/\sqrt{t}$ and $\abs{r-\tilde r}\leq R/\sqrt{t}$. Applying the union bound on $S$ yields
\begin{equation*}
\P_{r^*_{1:t-1}}\spar{\max_{(\alpha,r)\in S}\abs{\frac{1}{t-1}\sum_{i=1}^{t-1}l_\alpha(r,r^*_i)-\E_{r^*_t}[l_\alpha(r,r^*_t)]}\geq \eps}\leq 2t\exp\rpar{-\frac{2(t-1)\eps^2}{R^2}},
\end{equation*}
which means with probability at least $1-\delta$, 
\begin{equation*}
\max_{(\alpha,r)\in S}\abs{\frac{1}{t-1}\sum_{i=1}^{t-1}l_\alpha(r,r^*_i)-\E_{r^*_t}[l_\alpha(r,r^*_t)]}\leq R\sqrt{\frac{\log(2t/\delta)}{2(t-1)}}.
\end{equation*}
Since $l_\alpha(r,r^*)$ is $R$-Lipschitz with respect to $\alpha$, and $1$-Lipschitz with respect to $r$, we have
\begin{equation*}
\max_{0\leq \alpha\leq 1,0\leq r\leq R}\abs{\frac{1}{t-1}\sum_{i=1}^{t-1}l_\alpha(r,r^*_i)-\E_{r^*_t}[l_\alpha(r,r^*_t)]}\leq R\sqrt{\frac{\log(2t/\delta)}{2(t-1)}}+\frac{2R}{\sqrt{t}}.
\end{equation*}

Finally, due to Theorem~\ref{thm:main} we have for all 
$\alpha$ and $r$, 
\begin{align*}
\frac{1}{t-1}\sum_{i=1}^{t-1}l_\alpha(r_t(\alpha),r^*_i)&\leq \frac{1}{t-1}\sum_{i=1}^{t-1}l_\alpha(r,r^*_i)+\frac{\lambda_t}{1-\lambda_t}\spar{\psi(r)-\psi(r_t(\alpha))}\\
&\leq \frac{1}{t-1}\sum_{i=1}^{t-1}l_\alpha(r,r^*_i)+\frac{1}{\sqrt{t}-1}\max\left\{\psi(0),\psi(R)\right\}\\
&\leq \frac{1}{t-1}\sum_{i=1}^{t-1}l_\alpha(r,r^*_i)+\frac{R}{\sqrt{t}-1}.
\end{align*}
Combining it with the generalization error bound above, with high probability we have for all $\alpha$ and $r$,
\begin{equation*}
\E_{r^*_t}[l_\alpha(r_t(\alpha),r^*_t)]\leq \E_{r^*_t}[l_\alpha(r,r^*_t)]+\frac{R}{\sqrt{t}-1}+2\rpar{R\sqrt{\frac{\log(2t/\delta)}{2(t-1)}}+\frac{2R}{\sqrt{t}}}.
\end{equation*}
Taking $\min_r$ on the RHS completes the proof.
\end{proof}

\quantized*

\begin{proof}[Proof of Theorem~\ref{thm:quantized}]
Recall from Section~\ref{subsection:quantized} that the quantized true score is denoted by $\tilde r^*_t$. From Theorem~\ref{thm:regret}, we have
\begin{align*}
\sum_{t=1}^Tl_\alpha(r_t(\alpha),\tilde r^*_t)-\sum_{t=1}^Tl_\alpha(q_{\alpha}(r^*_{1:T}),\tilde r^*_t)&\leq \sum_{t=1}^Tl_\alpha(r_t(\alpha),\tilde r^*_t)-\sum_{t=1}^Tl_\alpha(q_{\alpha}(\tilde r^*_{1:T}),\tilde r^*_t)\\
&=O(R\sqrt{T}). 
\end{align*}
As $\abs{\tilde r^*_t-r^*_t}\leq R/\sqrt{T}$ and the quantile loss function $l_\alpha(r,r^*)$ is $1$-Lipschitz with respect to $r^*$, we have
\begin{align*}
\abs{\sum_{t=1}^Tl_\alpha(r_t(\alpha),\tilde r^*_t)-\sum_{t=1}^Tl_\alpha(r_t(\alpha),r^*_t)}&\leq \sum_{t=1}^T\abs{l_\alpha(r_t(\alpha),\tilde r^*_t)-l_\alpha(r_t(\alpha),r^*_t)}\\
&\leq \sum_{t=1}^T\abs{\tilde r^*_t-r^*_t}\leq R\sqrt{T}.
\end{align*}
The comparator term $\sum_{t=1}^Tl_\alpha(q_{\alpha}(r^*_{1:T}),\tilde r^*_t)$ can be related similarly to $\sum_{t=1}^Tl_\alpha(q_{\alpha}(r^*_{1:T}),r^*_t)$, and combining the above completes the proof.
\end{proof}

\discounted*

\begin{proof}[Proof of Theorem~\ref{thm:discounted}]
Analogous to Eq.(\ref{eq:compare_to}), we can write the output of \textsc{Discounted} as
\begin{align*}
r_t(\alpha)&=q_{\alpha}\spar{\lambda P_0+(1-\lambda )\bar P_\beta (r^*_{1:t-1})}\\
&=\min\left\{r:\frac{r}{R}\spar{\lambda+\beta^{t-1}(1-\lambda)}+(1-\lambda)(1-\beta)\sum_{i=1}^{t-1}\beta^{t-1-i}\bm{1}[r^*_i\leq r]\geq \alpha\right\}.
\end{align*}
Similar to Eq.(\ref{eq:ftrl_objective}), this can be verified as a minimizer of the objective
\begin{equation*}
H_t(r)\defeq (1-\beta)^{-1}\rpar{\frac{\lambda\beta^{1-t}}{1-\lambda}+1}\psi(r)+\sum_{i=1}^{t-1}\beta^{-i}l_\alpha(r,r^*_i).
\end{equation*}
For the convenience of notation, we will write the regularizer weight as $z_t\defeq (1-\beta)^{-1}\rpar{\frac{\lambda\beta^{1-t}}{1-\lambda}+1}$.

Notice that with the uniform $P_0$, the base regularizer $\psi$ is $R^{-1}$-strongly-convex due to Theorem~\ref{thm:main}, therefore we can apply the strong-convexity-based FTRL analysis \citep[Corollary~7.9]{orabona2023modern} on the scaled loss functions,
\begin{equation*}
h_t(r)\defeq\beta^{-t}l_\alpha(r,r^*_t).
\end{equation*}
This yields
\begin{equation*}
\sum_{t=1}^Th_t(r_t(\alpha))-\min_{r\in[0,R]}\sum_{t=1}^Th_t(r)\leq z_T\spar{\max_{r\in[0,R]}\psi(r)-\min_{r\in[0,R]}\psi(r)}+\frac{R}{2}\sum_{t=1}^T\frac{g_t^2}{z_t^2},
\end{equation*}
where $g_t$ can be any subgradient of $h_t(r)$ at $r=r_t(\alpha)$. Scaling both sides by $\beta^T$, we recover the discounted regret definition on the LHS:
\begin{equation*}
\reg_{T,\beta}(\alpha)\leq \beta^Tz_T\spar{\max_{r\in[0,R]}\psi(r)-\min_{r\in[0,R]}\psi(r)}+\frac{R\beta^T}{2}\sum_{t=1}^T\frac{g_t^2}{z_t}.
\end{equation*}

Next we simplify the obtained expression. The range of $\phi$ is contained in $[0,R/2]$. In addition, $\abs{g_t}\leq \beta^{-t}$ since the quantile loss $l_\alpha(r,r^*)$ is $1$-Lipschitz with respect to $r$. Therefore, 
\begin{equation*}
\sum_{t=1}^T\frac{g_t^2}{z_t}\leq \sum_{t=1}^T\frac{\beta^{-2t}}{(1-\beta)^{-1}\rpar{\frac{\lambda\beta^{1-t}}{1-\lambda}+1}}\leq \frac{1-\beta}{\beta}\frac{1-\lambda}{\lambda}\sum_{t=1}^T\beta^{-t}\leq \frac{1-\lambda}{\lambda}\beta^{-T-1},
\end{equation*}
\begin{equation*}
\reg_{T,\beta}(\alpha)\leq \frac{R}{2}\rpar{\frac{\beta^T}{1-\beta}+\frac{\lambda}{1-\lambda}\frac{\beta}{1-\beta}+\frac{1-\lambda}{\lambda}\frac{1}{\beta}}.
\end{equation*}
Notice that our choice of $\lambda$ satisfies $\frac{\lambda}{1-\lambda}=\beta^{-1}\sqrt{1-\beta}$, therefore
\begin{equation*}
\reg_{T,\beta}(\alpha)\leq \frac{R}{2}\rpar{\frac{\beta^T}{1-\beta}+\frac{2}{\sqrt{1-\beta}}}=\frac{R}{\sqrt{1-\beta}}+o(R),
\end{equation*}
where $o(\cdot)$ is with respect to $T\rightarrow \infty$.
\end{proof}

\subsection{Auxiliary Lemma}

\begin{lemma}[Lemma~7.6 of \citep{orabona2023modern}]\label{lemma:strong_convexity}
Let $f$ be a $\mu$-strongly convex function with respect to a norm $\norm{\cdot}$, over a convex set $V$. For all $x,y\in V$ and subgradients $g\in\partial f(y)$, $g'\in\partial f(x)$, we have
\begin{equation*}
f(x)-f(y)\leq \inner{g}{x-y}+\frac{1}{2\mu}\norm{g-g'}_*^2.
\end{equation*}
Here $\inner{\cdot}{\cdot}$ denotes the inner product, and $\norm{\cdot}_*$ denotes the dual norm of $\norm{\cdot}$. 
\end{lemma}

The following lemma is a standard tool in ML due to \citep{hoeffding1963probability}.

\begin{lemma}[Hoeffding's inequality]\label{lemma:hoeffding}
Let $x_1,\ldots,x_n$ be iid samples of a real-valued random variable on $[a,b]$. Let $\bar x$ be the mean of the distribution. Then, for all $\eps>0$, we have
\begin{equation*}
\P\spar{\abs{\frac{1}{n}\sum_{i=1}^nx_i-\bar x}\geq \eps}\leq 2\exp\rpar{-\frac{2n\eps^2}{(b-a)^2}}.
\end{equation*}
\end{lemma}

The next lemma is the celebrated Dvoretzky–Kiefer–Wolfowitz inequality, due to \citep{dvoretzky1956asymptotic,massart1990tight}.

\begin{lemma}[DKW inequality]\label{lemma:dkw}
Let $x_1,\ldots,x_n$ be iid samples of a real-valued random variable with cumulative distribution function $F$, and let $\bar P(x_{1:n})$ be the empirical distribution of $x_{1:n}$, with cumulative distribution function $\hat F_n$. For all $\eps>0$, we have
\begin{equation*}
\P\spar{\sup_{x\in\R}\abs{\hat F_n(x)-F(x)}>\eps}\leq 2\exp(-2n\eps^2).
\end{equation*}
\end{lemma}

\section{Additional experiment}\label{section:additional_experiment}

Extending Section~\ref{section:experiment}, this section presents the result of our stock price experiment using a different dataset (NVDA instead of AMD). The same procedure from Section~\ref{section:experiment} is followed. Figure~\ref{fig:stock_threshold_nvda} plots the predicted thresholds, and Figure~\ref{fig:loss_nvda} plots the total quantile loss. Overall they exhibit the similar behavior as the result from Section~\ref{section:experiment}. 

\begin{figure}[ht]
    \centering
    \includegraphics[width=0.9\linewidth]{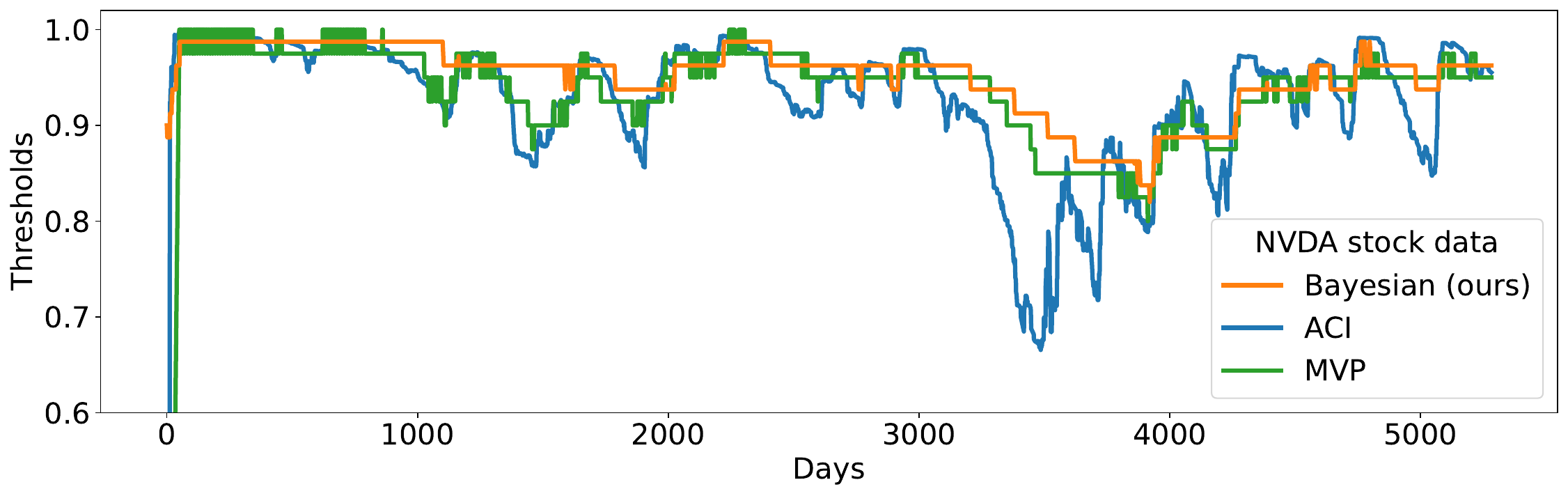}
    \caption{Predicted score threshold on NVDA stock data.}
    \label{fig:stock_threshold_nvda}
\end{figure}

\begin{figure}
\centering
    \begin{minipage}{0.36\textwidth}
        \includegraphics[width=\textwidth]{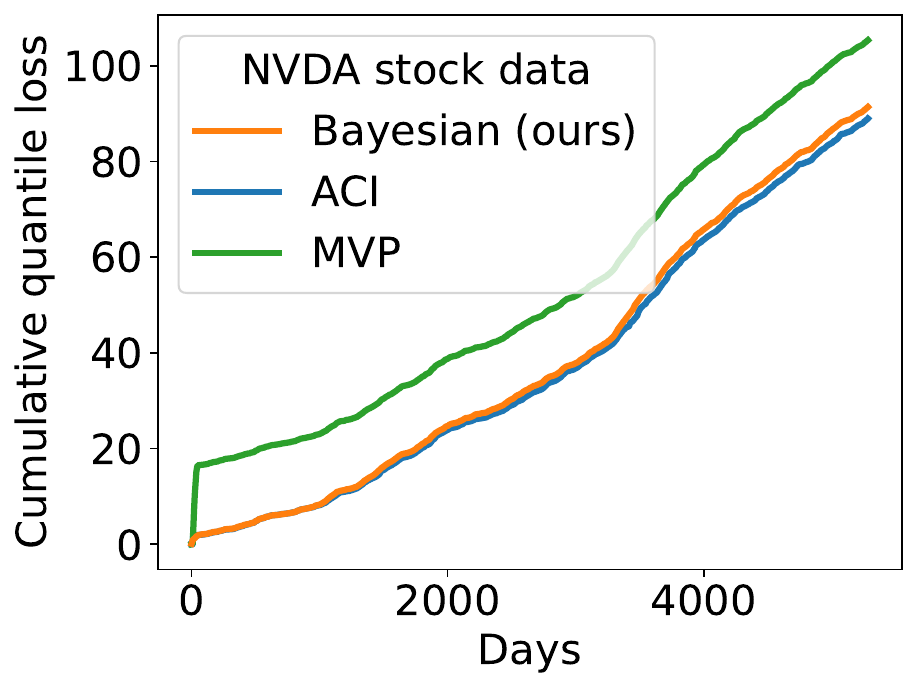}
    \end{minipage}
    \begin{minipage}{0.35\textwidth}
        \includegraphics[width=\textwidth]{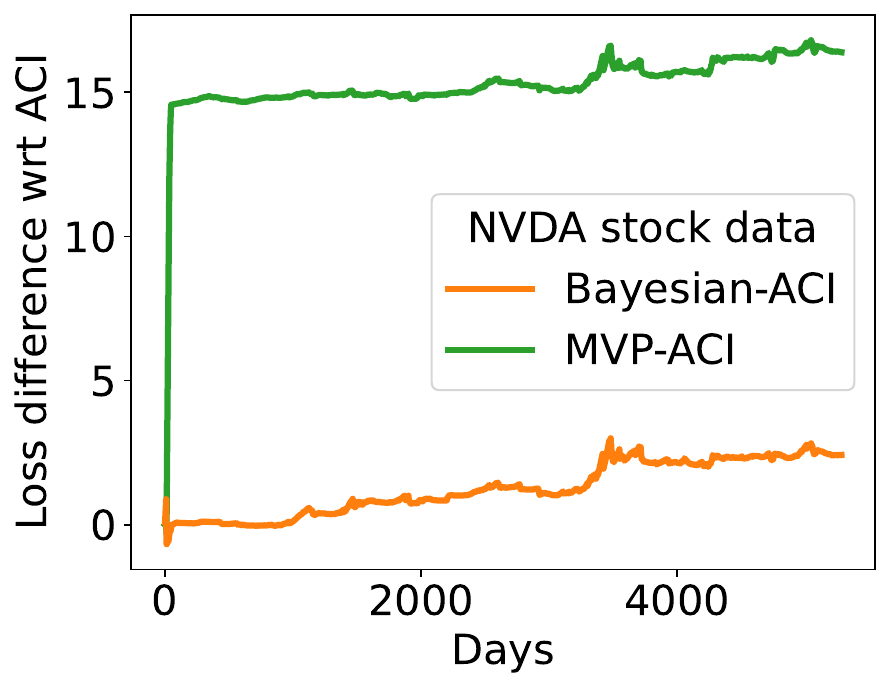}
    \end{minipage}
    \caption{Quantile loss on NVDA stock data.}
    \label{fig:loss_nvda}
\end{figure}

As for the coverage frequency, ACI achieves $0.899$, MVP achieves $0.891$, and our Bayesian algorithm achieves $0.897$. Again, closer to the target $0.9$ is better. The conclusion is that in the fixed-$\alpha$ setting our algorithm performs competitively compared to the baselines, while in the multi-$\alpha$ setting it demonstrates the advantage from Section~\ref{section:validity}. 

\end{document}